%% file: main.tex
\documentclass[letterpaper]{article}
\usepackage{uai2019new}
\usepackage[margin=1in]{geometry}
\usepackage[round]{natbib}
\usepackage{amsmath,amssymb,amsthm}
\usepackage{float}
\usepackage{graphicx}
\usepackage{mathrsfs}

\usepackage{algorithm}
\usepackage{booktabs}
\usepackage{hyperref}
\usepackage{algorithmic}
\newtheorem{theorem}{Theorem}

\newtheorem{corollary}{Corollary}[theorem]
\newtheorem{lemma}{Lemma}[theorem]
\newtheorem{definition}{Definition}

\usepackage{fancyhdr}

\DeclareMathOperator*{\argmin}{arg\,min}
\def\IB{\text{IB}}
\def\E{\mathbb{E}}
\def\log{\text{log}}
\def\D{\mathcal{D}}
\def\X{\mathcal{X}}
\def\KL{\operatorname{KL}}
\def\O{\mathcal{O}}
\usepackage{bbm}

\usepackage{times}

\title{Learnability for the Information Bottleneck}

\author{} 

\author{ {\bf Tailin Wu
} \\
MIT\\
\texttt{tailin@mit.edu}\\
\And
{\bf Ian Fischer}  \\
Google Research          \\
\texttt{iansf@google.com}\\
\And
{\bf Isaac L. Chuang}   \\
MIT \\
\texttt{ichuang@mit.edu}  \\
\And
{\bf Max Tegmark}  \\
MIT          \\
\texttt{tegmark@mit.edu}  \\
}

\begin{document}

\maketitle

\begin{abstract}
The Information Bottleneck (IB) method (\cite{tishby2000information}) provides an insightful and principled approach for balancing compression and prediction for representation learning. 
The IB objective $I(X;Z)-\beta I(Y;Z)$ employs a Lagrange multiplier $\beta$ to tune this trade-off. 
However, in practice, not only is $\beta$ chosen empirically without theoretical guidance, there is also a lack of theoretical understanding between $\beta$, learnability, the intrinsic nature of the dataset and model capacity. 
In this paper, we show that if $\beta$ is improperly chosen, learning cannot happen -- the trivial representation $P(Z|X)=P(Z)$ becomes the global minimum of the IB objective. 
We show how this can be avoided, by identifying a sharp phase transition between the unlearnable and the learnable which arises as $\beta$ is varied.
This phase transition defines the concept of IB-Learnability.
We prove several sufficient conditions for IB-Learnability, which provides theoretical guidance for choosing a good $\beta$. 
We further show that IB-learnability is determined by the largest \emph{confident}, \emph{typical}, and \emph{imbalanced subset} of the examples (the \textit{conspicuous subset}),
and discuss its relation with model capacity. 
We give practical algorithms to estimate the minimum $\beta$ for a given dataset. 
We also empirically demonstrate our theoretical conditions with analyses of synthetic datasets, MNIST, and CIFAR10.

\end{abstract}

\section{INTRODUCTION}

\citet{tishby2000information} introduced the \textit{Information Bottleneck} (IB) objective function which learns a representation $Z$ of observed variables $(X,Y)$ that retains as little information about $X$ as possible, but simultaneously captures as much information about $Y$ as possible:
\begin{equation}
\label{eq:IB_beta}
\min \IB_\beta(X,Y;Z) = \min [I(X;Z) - \beta I(Y;Z)]
\end{equation}
$I(\cdot)$ is the mutual information.
The hyperparameter $\beta$ controls the trade-off between compression and prediction, in the same spirit as Rate-Distortion Theory~\citep{shannon}, but with a learned representation function $P(Z|X)$ that automatically captures some part of the ``semantically meaningful'' information, where the semantics are determined by the observed relationship between $X$ and $Y$.

The IB framework has been extended to and extensively studied in a variety of scenarios, including Gaussian variables (\cite{chechik2005information}), meta-Gaussians (\cite{rey2012meta}), continuous variables via variational methods (\cite{alemi2016deep,chalk2016relevant,fischer2018the}), deterministic scenarios (\cite{strouse2017deterministic,kolchinsky2018caveats}), geometric clustering (\cite{strouse2017information}), and is used for learning invariant and disentangled representations in deep neural nets (\cite{achille2018emergence,achille2018information}).
However, a core issue remains: how should we set a good $\beta$? 
In the original work, the authors recommend sweeping $\beta > 1$, which can be prohibitively expensive in practice, but also leaves open interesting theoretical questions around the relationship between $\beta$, $P(Z|X)$, and the observed data, $P(X,Y)$.

This work begins to answer some of those questions by characterizing the \textit{onset} of learning.
Specifically:
\begin{itemize}
\item We show that improperly chosen $\beta$ may result in a failure to learn: the trivial solution $P(Z|X) = P(Z)$ becomes the global minimum of the IB objective, even for $\beta \gg 1$ (Section~\ref{sec:motivation}).
\item We introduce the concept of \textit{IB-Learnability}, and show that when we vary $\beta$, the IB objective will undergo a phase transition from the inability to learn to the ability to learn (Section~\ref{sec:learnability}).
\item Using the second-order variation, we derive sufficient conditions for IB-Learnability, which provide theoretical guidance for choosing a good $\beta$ (Section~\ref{sec:suff_conditions}).
\item We show that IB-Learnability is determined by the largest \emph{confident}, \emph{typical}, and \emph{imbalanced subset} of the examples (the \textit{conspicuous subset}), reveal its relationship with the slope of the Pareto frontier at the origin on the information plane $I(X;Z)$ vs. $I(Y;Z)$, and discuss its relation to model capacity (Section~\ref{sec:discussion}).
\item We additionally prove a deep relationship between IB-Learnability, the hypercontractivity coefficient, the contraction coefficient, and the maximum correlation (Section~\ref{sec:discussion}).
\end{itemize}

We also present an algorithm for estimating the onset of IB-Learnability and the conspicuous subset, and demonstrate that it does a good job of 
approximating both the theoretical predictions and the empirical results (Section~\ref{sec:estimate}).
Finally, we use our main results to demonstrate on synthetic datasets, MNIST~\citep{mnist} and CIFAR10 \citep{cifar} that the theoretical prediction for IB-Learnability closely matches experiment (Section~\ref{sec:experiments}).

\subsection{A Motivating Example}
\label{sec:motivation}

How can we choose a good $\beta$?
To gain intuition, consider learning multiple Variational Information Bottleneck (VIB) representations~\citep{alemi2016deep} of MNIST~\citep{mnist} at different $\beta$.
We select the digits 0 and 1 for binary classification, and add class-conditional noise~\citep{angluin1988learning} to the labels with flip probability 0.2, which simulates a general scenario where the data may be noisy and the dependence of $Y$ on $X$ is not deterministic.
The algorithm only sees the corrupted labels.
Fig.~\ref{fig:mnist_0.2} shows the converged accuracy on the true labels for the VIB models plotted against $\beta$.
We see clearly that when $\beta < 3.25$, no learning happens, and the accuracy is the same as random guessing.
Beginning with $\beta > 3.25$, there is a clear phase transition where the accuracy sharply increases, indicating the objective is able to learn a non-trivial representation.
This kind of phase transition is typical in our experiments in Section~\ref{sec:experiments}.
When the noise rate is high, the transition can happen at $\beta \sim 500$; i.e., we need a large ``$\beta$ force'' to extract relevant information from $X$ to predict $Y$.
In that case, an improperly-chosen $\beta$ in the unlearnable region will preclude learning a useful representation.

\begin{figure}[t]
\begin{center}
\includegraphics[width=0.9\columnwidth]{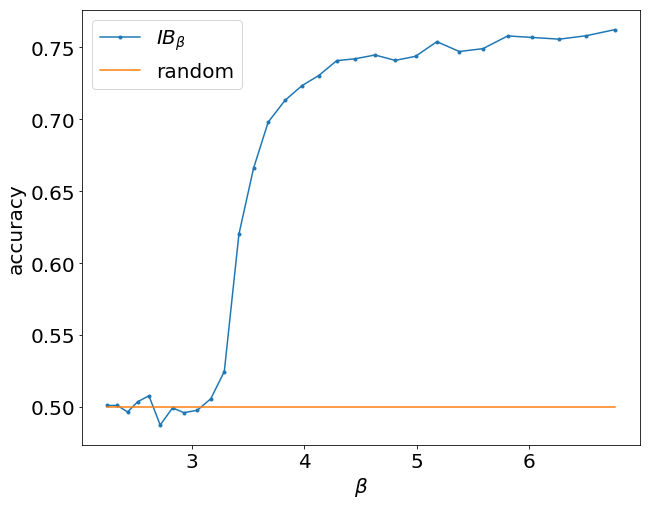}
\end{center}
\caption{Accuracy for binary classification of MNIST digits 0 and 1 with 20\% label noise and varying $\beta$. No learning happens for models trained at $\beta < 3.25$.} 
\label{fig:mnist_0.2}
\end{figure}

\section{RELATED WORK}
\label{sec:related_work}

The original IB work~\citep{tishby2000information} provides a tabular method for exactly computing the optimal encoder distribution $P(Z|X)$ for a given $\beta$ and cardinality of the discrete representation, $|Z|$.
Thus, the search for the desired model involves not only sweeping $\beta$, but also considering different representation dimensionalities.
These restrictions were lifted somewhat by~\citet{chechik2005information}, which presents the Gaussian Information Bottleneck (GIB) for learning a multivariate Gaussian representation $Z$ of $(X,Y)$, assuming that both $X$ and $Y$ are also multivariate Gaussians.
They also note the presence of the trivial solution not only when $\beta \le 1$, but also depending on the eigenspectrum of the observed variables.
However, the restriction to multivariate Gaussian datasets limits the generality of the analysis.
Another analytic treatment of IB is given in~\cite{rey2012meta}, which reformulates the objective in terms of the copula functions.
As with the GIB approach, this formulation restricts the form of the data distributions -- the copula functions for the joint distribution $(X,Y)$ are assumed to be known, which is unlikely in practice.

\citet{strouse2017deterministic} presents the Deterministic Information Bottleneck (DIB), which minimizes the coding cost of the representation, $H(Z)$, rather than the transmission cost, $I(X;Z)$ as in IB.
This approach learns hard clusterings with different code entropies that vary with $\beta$.
In this case, it is clear that a hard clustering with minimal $H(Z)$ will result in a single cluster for all of the data, which is the DIB trivial solution.
No analysis is given beyond this fact to predict the actual onset of learnability, however.

The first amortized IB objective is in the Variational Information Bottleneck (VIB) of~\citet{alemi2016deep}.
VIB replaces the exact, tabular approach of IB with variational approximations of the classifier distribution ($P(Y|Z)$) and marginal distribution ($P(Z)$).
This approach cleanly permits learning a stochastic encoder, $P(Z|X)$, that is applicable to any $x \in \mathcal{X}$, rather than just the particular $X$ seen at training time.
The cost of this flexibility is the use of variational approximations that may be less expressive than the tabular method.
Nevertheless, in practice, VIB learns easily and is simple to implement, so we rely on VIB models for our experimental confirmation.

Closely related to IB is the recently proposed Conditional Entropy Bottleneck (CEB)~\citep{fischer2018the}.
CEB attempts to explicitly learn the Minimum Necessary Information (MNI), defined as the point in the information plane where $I(X;Y) = I(X;Z) = I(Y;Z)$.
The MNI point may not be achievable even in principle for a particular dataset.
However, the CEB objective provides an explicit estimate of how closely the model is approaching the MNI point by observing that a necessary condition for reaching the MNI point occurs when $I(X;Z|Y) = 0$.
The CEB objective $I(X;Z|Y) - \gamma I(Y;Z)$ is equivalent to IB at $\gamma = \beta + 1$, so our analysis of IB-Learnability applies equally to CEB.

\citet{kolchinsky2018caveats} presents analytic and empirical results about trivial solutions in the particular setting of $Y$ being a deterministic function of $X$ in the observed sample.
However, their use of the term ``trivial solution'' is distinct from ours.
They are referring to the observation that $\beta$ will demonstrate trivial interpolation between two different but valid solutions on the optimal frontier, rather than demonstrating a non-trivial trade-off between compression and prediction as expected when varying the IB Lagrangian.
Our use of ``trivial'' refers to whether IB is capable of learning at all given a certain dataset and value of $\beta$.

\citet{achille2018information} apply the IB Lagrangian to the weights of a neural network, yielding InfoDropout.
In~\citet{achille2018emergence}, the authors give a deep and compelling analysis of how the IB Lagrangian can yield invariant and disentangled representations.
They do not, however, consider the question of the onset of learning, although they are aware that not all models will learn a non-trivial representation.
More recently, \citet{achille2018dynamics} repurpose the InfoDropout IB Lagrangian as a Kolmogorov Structure Function to analyze the ease with which a previously-trained network can be fine-tuned for a new task.
While that work is tangentially related to learnability, the question it addresses is substantially different from our investigation of the onset of learning.

Our work is also closely related to the hypercontractivity coefficient (\cite{anantharam2013maximal,polyanskiy2017strong}), defined as $\sup_{Z-X-Y}\frac{I(Y;Z)}{I(X;Z)}$, which by definition equals the inverse of $\beta_0$, our IB-learnability threshold.
In \cite{anantharam2013maximal}, the authors prove that the hypercontractivity cofficient equals the contraction coefficient $\eta_{\KL}(P_{Y|X}, P_X)$, and \citet{kim2017discovering} propose a practical algorithm to estimate $\eta_{\KL}(P_{Y|X},P_X)$, which provides a measure for potential influence in the data.
Although our goal is different, the sufficient conditions we provide for IB-Learnability are also lower bounds for the hypercontractivity coefficient.

\section{IB-LEARNABILITY}
\label{sec:learnability}

We are given instances of $(x,y)\in\mathcal{X}\times\mathcal{Y}$ drawn from a distribution with probability (density) $P(X,Y)$, where unless otherwise stated, both $X$ and $Y$ can be discrete or continuous variables.
$(X,Y)$ is our \textit{training data}, and may be characterized by different types of noise.
The nature of this training data and the choice of $\beta$ will be sufficient to predict the transition from unlearnable to learnable.

We can learn a representation $Z$ of $X$ with conditional probability\footnote{%
  We use capital letters $X,Y,Z$ for random variables and lowercase $x,y,z$ to denote the instance of variables, with $P(\cdot)$ and $p(\cdot)$ denoting their probability or probability density, respectively.
}
$p(z|x)$, such that $X,Y,Z$ obey the Markov chain $Z\gets X\leftrightarrow Y$.
Eq.~\ref{eq:IB_beta} above gives the IB objective with Lagrange multiplier $\beta$, $\IB_\beta(X,Y;Z)$, which is a functional of $p(z|x)$: $\IB_\beta(X,Y;Z)=\IB_\beta[p(z|x)]$.
The IB learning task is to find a conditional probability $p(z|x)$ that minimizes $\IB_\beta(X,Y;Z)$.
The larger $\beta$, the more the objective favors making a good prediction for $Y$.
Conversely, the smaller $\beta$, the more the objective favors learning a concise representation.

How can we select $\beta$ such that the IB objective learns a useful representation?
In practice, the selection of $\beta$ is done empirically.
Indeed,~\citet{tishby2000information} recommends ``sweeping $\beta$''.
In this paper, we provide theoretical guidance for choosing $\beta$ by introducing the concept of $\IB$-Learnability and providing a series of $\IB$-learnable conditions.

\begin{definition}
\label{def:learnable}
$(X,Y)$ is $\IB_\beta$-learnable if there exists a $Z$ given by some $p_1(z|x)$, such that $\IB_\beta(X,Y;Z)\rvert_{p_1(z|x)} < \IB_\beta(X,Y;Z)\rvert_{p(z|x)=p(z)}$, where $p(z|x)=p(z)$ characterizes the trivial representation where 
$Z=Z_\text{trivial}$ is independent of $X$.
\end{definition}

If $(X;Y)$ is $\IB_\beta$-learnable, then when $\IB_\beta(X,Y;Z)$ is globally minimized, it will \emph{not} learn a trivial representation.
On the other hand, if $(X;Y)$ is not $\IB_\beta$-learnable, then when $\IB_\beta(X,Y;Z)$ is globally minimized, it may learn a trivial representation.

\paragraph{Trivial solutions.}
Definition~\ref{def:learnable} defines trivial solutions in terms of representations where $I(X;Z) = I(Y;Z) = 0$.
Another type of trivial solution occurs when $I(X;Z) > 0$ but $I(Y;Z) = 0$.
This type of trivial solution is not directly achievable by the IB objective, as $I(X;Z)$ is minimized, but it can be achieved by construction or by chance.
It is possible that starting learning from $I(X;Z) > 0, I(Y;Z) = 0$ could result in access to non-trivial solutions not available from $I(X;Z) = 0$.
We do not attempt to investigate this type of trivial solution in this work.

\paragraph{Necessary condition for IB-Learnability.}
From Definition~\ref{def:learnable}, we can see that $\IB_\beta$-Learnability for any dataset $(X;Y)$ requires $\beta > 1$.
In fact, from the Markov chain $Z\gets X\leftrightarrow Y$, we have $I(Y;Z) \le I(X;Z)$ via the data-processing inequality.
If $\beta \le 1$, then since $I(X;Z) \ge 0$ and $I(Y;Z) \ge 0$, we have that $\min(I(X;Z) - \beta I(Y;Z))=0 = \IB_\beta(X,Y;Z_{trivial})$.
Hence $(X,Y)$ is not $\IB_\beta$-learnable for $\beta \le 1$.

Due to the reparameterization invariance of mutual information, we have the following theorem for $\IB_\beta$-Learnability:

\begin{theorem}
\label{thm:homo_learnability}
Let $X'=g(X)$ be an invertible map (if $X$ is a continuous variable, $g$ is additionally required to be continuous). Then $(X,Y)$ and $(X',Y)$ have the same $\IB_\beta$-Learnability.
\end{theorem}

The proof for Theorem \ref{thm:homo_learnability} is in Appendix \ref{app:homo_learnability}.
Theorem \ref{thm:homo_learnability} implies a favorable property for any condition for $\IB_\beta$-Learnability: the condition should be invariant to invertible mappings of $X$.
We will inspect this invariance in the conditions we derive in the following sections.

\section{SUFFICIENT CONDITIONS FOR IB-LEARNABILITY}
\label{sec:suff_conditions}

Given $(X,Y)$, how can we determine whether it is $\IB_\beta$-learnable?
To answer this question, we derive a series of sufficient conditions for $\IB_\beta$-Learnability, starting from its definition.
The conditions are in increasing order of practicality, while sacrificing as little generality as possible.

Firstly, Theorem \ref{thm:beta_monotonic} characterizes the $\IB_\beta$-Learnability range for $\beta$, with proof in Appendix \ref{app:beta_monotonic}:

\begin{theorem}
\label{thm:beta_monotonic}
If $(X,Y)$ is $\IB_{\beta_1}$-learnable, then for any $\beta_2>\beta_1$, it is $\IB_{\beta_2}$-learnable.
\end{theorem}

Based on Theorem \ref{thm:beta_monotonic}, the range of $\beta$ such that $(X,Y)$ is $\IB_\beta$-learnable has the form $\beta \in (\beta_0, +\infty)$.
Thus, $\beta_0$ is the \textit{threshold} of IB-Learnability.

\begin{lemma}
\label{lemma:stationary}
$p(z|x)=p(z)$ is a stationary solution for $\IB_\beta(X,Y;Z)$.
\end{lemma}
The proof in Appendix \ref{app:stationary} shows that both first-order variations
$\delta I(X;Z)=0$ and $\delta I(Y;Z)=0$ vanish at the trivial representation $p(z|x)=p(z)$, so
$\delta \IB_\beta[p(z|x)] = 0$ at the trivial representation.

Lemma \ref{lemma:stationary} yields our strategy for finding sufficient conditions for learnability: find conditions such that $p(z|x)=p(z)$ is not a local minimum for the functional $\IB_\beta[p(z|x)]$. 
Based on the necessary condition for the minimum (Appendix \ref{app:suff_1}), we have the following theorem \footnote{The theorems in this paper deal with learnability w.r.t. true mutual information. If parameterized models are used to approximate the mutual information, the limitation of the model capacity will translate into more uncertainty of $Y$ given $X$, viewed through the lens of the model.}:

\begin{theorem}[\textbf{Suff. Cond. 1}]
\label{thm:suff_1}
A sufficient condition for $(X, Y)$ to be $\IB_\beta$-learnable is that there exists a perturbation function\footnote{so that the perturbed probability (density) is $p'(z|x)=p(z|x)+\epsilon\cdot h(z|x)$. Also, for integrals, whenever a variable $W$ is discrete, we can simply replace the integral $(\int \cdot dw)$ by summation $(\sum_w\cdot)$.} $h(z|x)$ with
$\int h(z|x)dz = 0$, such that the second-order variation $\delta^2 \IB_\beta[p(z|x)] < 0$ at the trivial representation $p(z|x)=p(z)$.
\end{theorem}

The proof for Theorem  \ref{thm:suff_1} is given in Appendix \ref{app:suff_1}.
Intuitively, if $\delta^2 \IB_\beta[p(z|x)]\big\rvert_{p(z|x)=p(z)} < 0$, we can always find a $p'(z|x) = p(z|x) + \epsilon \cdot h(z|x)$ in the neighborhood of the trivial representation $p(z|x)=p(z)$, such that $\IB_\beta[p'(z|x)] < \IB_\beta[p(z|x)]$, thus satisfying the definition for $\IB_\beta$-Learnability.

To make Theorem \ref{thm:suff_1} more practical, we perturb $p(z|x)$ around the trivial solution $p'(z|x) = p(z|x) + \epsilon \cdot h(z|x)$, and expand $\IB_\beta[p(z|x) + \epsilon\cdot h(z|x)] - \IB_\beta[p(z|x)]$ to the second order of $\epsilon$.
We can then prove Theorem \ref{thm:suff_2}:

\begin{theorem}[\textbf{Suff. Cond. 2}]
\label{thm:suff_2}
A sufficient condition for $(X,Y)$ to be $\IB_\beta$-learnable is $X$ and $Y$ are not independent, and
\begin{equation}
\begin{aligned}
\label{eq:suff_2}
\beta > \inf_{h(x)}
\beta_0[h(x)]
\end{aligned}
\end{equation}

where the functional $\beta_0[h(x)]$ is given by
$$\beta_0[h(x)]=\frac{\E_{x \sim p(x)} [h(x)^2] - \left(\E_{x\sim p(x)} [h(x)]\right)^2}{\E_{y \sim p(y)}\left[\left(\E_{x \sim p(x|y)} [h(x)]\right)^2\right] - \left(\E_{x\sim p(x)} [h(x)]\right)^2}$$

Moreover, we have that $\left(\inf_{h(x)}\beta[h(x)]\right)^{-1}$ is a lower bound of the slope of the Pareto frontier in the information plane $I(Y;Z)$ vs. $I(X;Z)$ at the origin.
\end{theorem}

The proof is given in Appendix \ref{app:suff_2}, which also shows that if $\beta>\inf_{h(x)}\beta_0[h(x)]$ in Theorem \ref{thm:suff_2} is satisfied, we can construct a perturbation function $h(z|x)=h^*(x)h_2(z)$ with $h^*(x)=\argmin_{h(x)}\beta_0[h(x)]$, $\int h_2(z)dz=0, \int \frac{h_2^2(z)}{p(z)}dz>0$ for some $h_2(z)$, such that $h(z|x)$ satisfies Theorem \ref{thm:suff_1}. 
It also shows that the converse is true: if there exists $h(z|x)$ such that the condition in Theorem \ref{thm:suff_1} is true, then Theorem \ref{thm:suff_2} is satisfied\footnote{%
    We do not claim that any $h(z|x)$ satisfying Theorem~\ref{thm:suff_1} can be decomposed to $h^*(x)h_2(z)$ at the onset of learning.
    But from the equivalence of Theorems \ref{thm:suff_1} and \ref{thm:suff_2} as explained above, when there exists an $h(z|x)$ such that Theorem \ref{thm:suff_1} is satisfied, we can always construct an $h'(z|x)=h^*(x)h_2(z)$ that also satisfies Theorem \ref{thm:suff_1}.
}, i.e. $\beta>\inf_{h(x)}\beta_0[h(x)]$.
Moreover, letting the perturbation function $h(z|x)=h^*(x)h_2(z)$ at the trivial solution, we have

\begin{align}
\label{eq:what_first_learns}
p_\beta(y|x) = p(y) &+ \epsilon^2 C_z (h^*(x)-\overline{h}^*_x) \nonumber \\
&~\cdot \int p(x,y)(h^*(x)-\overline{h}^*_x)dx
\end{align}

where $p_\beta(y|x)$ is the estimated $p(y|x)$ by IB for a certain $\beta$, $\overline{h}^*_x=\int h^*(x)p(x)dx$, and $C_z=\int\frac{h_2^2(z)}{p(z)}dz>0$ is a constant.
This shows how the $p_\beta(y|x)$ by IB explicitly depends on $h^*(x)$ at the onset of learning.
The proof is provided in Appendix~\ref{app:what_first_learns}.

Theorem \ref{thm:suff_2} suggests a method to estimate $\beta_0$: we can parameterize $h(x)$ e.g. by a neural network, with the objective of minimizing $\beta_0[h(x)]$.
At its minimization, $\beta_0[h(x)]$ provides an upper bound for $\beta_0$,  and $h(x)$ provides a \emph{soft clustering} of the examples corresponding to a nontrivial perturbation of $p(z|x)$ at $p(z|x)=p(z)$ that minimizes $\IB_\beta[p(z|x)]$. 

Alternatively, based on the property of $\beta_0[h(x)]$, we can also use a specific functional form for $h(x)$ in Eq.~(\ref{eq:suff_2}), and obtain a stronger sufficient condition for $\IB_\beta$-Learnability.
But we want to choose $h(x)$ as near to the infimum as possible.
To do this, we note the following characteristics for the R.H.S of Eq.~(\ref{eq:suff_2}):

\begin{itemize}
\item We can set $h(x)$ to be nonzero if $x \in \Omega_x$ for some region $\Omega_x\subset \X$ and 0 otherwise.
Then we obtain the following sufficient condition:
\begin{equation}
\begin{aligned}
\label{eq:suff_2_omega}
\beta>\inf_{h(x),\Omega_x\in\X}\frac{\frac{\E_{x \sim p(x), x \in \Omega_x} [h(x)^2]}{\left(\E_{x \sim p(x), x \in \Omega_x} [h(x)]\right)^2} - 1}{\int\frac{dy}{p(y)}\left(\frac{\E_{x \sim p(x), x \in \Omega_x} [p(y|x) h(x)]}{\E_{x \sim p(x), x \in \Omega_x} [h(x)]}\right)^2 - 1}
\end{aligned}
\end{equation}

\item The numerator of the R.H.S. of Eq.~(\ref{eq:suff_2_omega}) attains its minimum when $h(x)$ is a constant within $\Omega_x$.
This can be proved using the Cauchy-Schwarz inequality: $\langle u,u \rangle \langle v,v \rangle \geq \langle u,v \rangle^2$, setting $u(x) = h(x) \sqrt{p(x)}$, $v(x) = \sqrt{p(x)}$, and defining the inner product as $\langle u,v \rangle = \int u(x) v(x) dx$.
Therefore, the numerator of the R.H.S. of Eq.~(\ref{eq:suff_2_omega}) $\ge \frac{1}{\int_{x \in \Omega_x} p(x)} - 1$, and attains equality when $\frac{u(x)}{v(x)} = h(x)$ is constant.
\end{itemize}

Based on these observations, we can let $h(x)$ be a nonzero constant inside some region $\Omega_x\subset\X$ and 0 otherwise, and the infimum over an arbitrary function $h(x)$ is simplified to infimum over $\Omega_x\subset\X$, and we obtain a sufficient condition for $\IB_\beta$-Learnability, which is a key result of this paper:

\begin{theorem}[\textbf{Conspicuous Subset Suff. Cond.}]
\label{thm:suff_3}
A sufficient condition for $(X,Y)$ to be $\IB_\beta$-learnable is $X$ and $Y$ are not independent, and
\begin{equation}
\begin{aligned}
\label{eq:suff_3}
\beta > \inf_{\Omega_x\subset \mathcal{X}}\beta_0(\Omega_x)
\end{aligned}
\end{equation}
where 
$$\beta_0(\Omega_x)=\frac{\frac{1}{p(\Omega_x)} - 1}{\E_{y \sim p(y|\Omega_x)} \left[ \frac{p(y|\Omega_x)}{p(y)} - 1 \right]}$$
$\Omega_x$ denotes the event that $x \in \Omega_x$, with probability $p(\Omega_x)$. 

$\left(\inf_{\Omega_x\subset\X}\beta_0(\Omega_x)\right)^{-1}$ gives a lower bound of the slope of the Pareto frontier in the information plane $I(Y;Z)$ vs. $I(X;Z)$ at the origin.
\end{theorem}

The proof is given in Appendix \ref{app:suff_3}. In the proof we also show that this condition is invariant to invertible mappings of $X$.

\section{Discussion}
\label{sec:discussion}

\paragraph{The conspicuous subset determines $\beta_0$.}

From Eq. (\ref{eq:suff_3}), we see that three characteristics of the subset $\Omega_x\subset\X$ lead to low $\beta_0$:
\textbf{(1) confidence:} $p(y|\Omega_x)$ is large;
\textbf{(2) typicality and size:} the number of elements in $\Omega_x$ is large, or the elements in $\Omega_x$ are typical, leading to a large probability of $p(\Omega_x)$;
\textbf{(3) imbalance:} $p(y)$ is small for the subset $\Omega_x$, but large for its complement.
In summary, $\beta_0$ will be determined by the largest \emph{confident}, \emph{typical} and \emph{imbalanced subset} of examples, or an equilibrium of those characteristics. We term $\Omega_x$ at the minimization of $\beta_0(\Omega_x)$ the \emph{conspicuous subset}.

\paragraph{Multiple phase transitions.}
Based on this characterization of $\Omega_x$, we can hypothesize datasets with multiple learnability phase transitions.
Specifically, consider a region $\Omega_{x0}$ that is small but ``typical'', consists of all elements confidently predicted as $y_0$ by $p(y|x)$, and where $y_0$ is the least common class.
By construction, this $\Omega_{x0}$ will dominate the infimum in Eq.~(\ref{eq:suff_3}), resulting in a small value of $\beta_0$.
However, the remaining $\mathcal{X} - \Omega_{x0}$ effectively form a new dataset, $\mathcal{X}_1$.
At exactly $\beta_0$, we may have that the current encoder, $p_0(z|x)$, has no mutual information with the remaining classes in $\mathcal{X}_1$; i.e., $I(Y_1;Z_0) = 0$.
In this case, Definition~\ref{def:learnable} applies to $p_0(z|x)$ with respect to $I(X_1;Z_1)$.
We might expect to see that, at $\beta_0$, learning will plateau until we get to some $\beta_1 > \beta_0$ that defines the phase transition for $\mathcal{X}_1$.
Clearly this process could repeat many times, with each new dataset $\mathcal{X}_i$ being distinctly more difficult to learn than $\mathcal{X}_{i-1}$.

\paragraph{Similarity to information measures.}
The denominator of $\beta_0(\Omega_x)$ in Eq.~(\ref{eq:suff_3}) is closely related to mutual information.
Using the inequality $x - 1 \ge \log(x)$ for $x > 0$, it becomes:
\begin{align*}
\E_{y \sim p(y|\Omega_x)} \bigg[\frac{p(y|\Omega_x)}{p(y)} - 1 \bigg] &\ge \E_{y \sim p(y|\Omega_x)} \bigg[ \log \frac{p(y|\Omega_x)}{p(y)} \bigg] \\
&= \tilde{I}(\Omega_x;Y)
\end{align*}

where $\tilde{I}(\Omega_x;Y)$ is the mutual information ``density'' at $\Omega_x\subset\X$.
Of course, this quantity is also $\mathbb{D}_{\KL}[p(y|\Omega_x)||p(y)]$, so we know that the denominator of Eq.~(\ref{eq:suff_3}) is non-negative. Incidentally, $\E_{y \sim p(y|\Omega_x)} \big[\frac{p(y|\Omega_x)}{p(y)} - 1 \big]$ is the density of ``rational mutual information'' (\cite{lin2016criticality}) at $\Omega_x$.

Similarly, the numerator of $\beta_0(\Omega_x)$ is related to the self-information of $\Omega_x$:
$$\frac{1}{p(\Omega_x)} - 1 \ge \log \frac{1}{p(\Omega_x)} = -\log\ p(\Omega_x) = h(\Omega_x)$$
so we can estimate the phase transition as:
\begin{equation}
\label{eq:info_approx}
\beta \gtrapprox \inf_{\Omega_x\subset \X} \frac{h(\Omega_x)}{\tilde{I}(\Omega_x;Y)}
\end{equation}
Since Eq.~(\ref{eq:info_approx}) uses upper bounds on both the numerator and the denominator, it does not give us a bound on $\beta_0$.

\paragraph{Estimating model capacity.}
The observation that a model can't distinguish between cluster overlap in the data and its own lack of capacity gives an interesting way to use IB-Learnability to measure the capacity of a set of models relative to the task they are being used to solve.

\paragraph{Learnability and the Information Plane.}

\begin{figure}[t]
\begin{center}
\includegraphics[width=0.9\columnwidth]{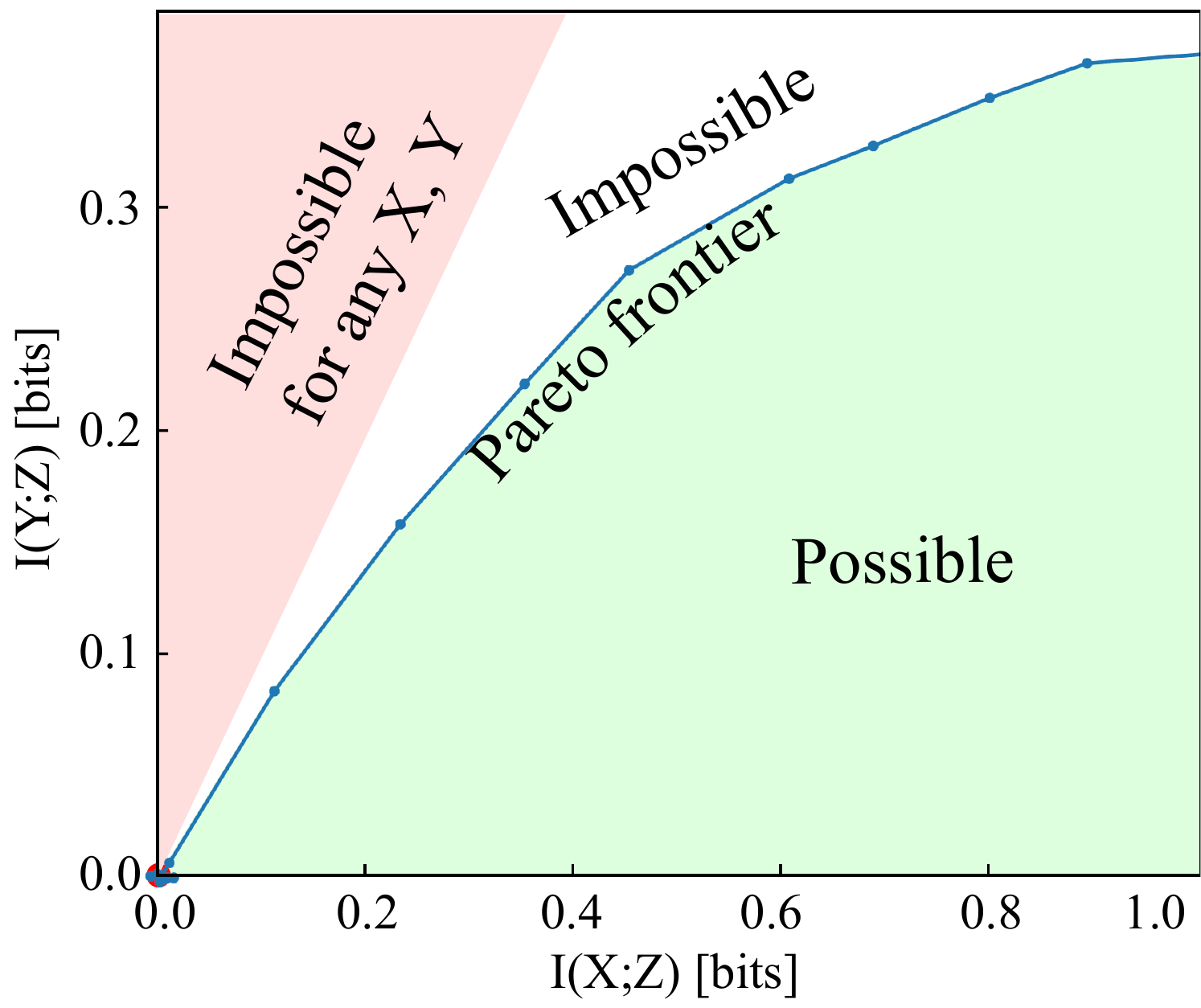}
\end{center}
\caption{The Pareto frontier of the information plane, $I(X;Z)$ vs $I(Y;Z)$, for the binary classification of MNIST digits 0 and 1 with 20\% label noise described in Sec.~\ref{sec:motivation} and Fig.~\ref{fig:mnist_0.2}.
For this problem, learning happens for models trained at $\beta > 3.25$.
$H(Y)=1$ bit since only two of ten digits are used, and $I(Y;Z) \le I(X;Y) \approx 0.5$ bits $<H(Y)$ because of the 20\% label noise.
The true frontier is differentiable; the figure shows a variational approximation that places an upper bound on both informations, horizontally offset to pass through the origin.
}
\label{fig:frontier}
\end{figure}

Many of our results can be interpreted in terms of the geometry of the Pareto frontier illustrated in Fig.~\ref{fig:frontier}, which describes the trade-off between increasing $I(Y;Z)$ and decreasing $I(X;Z)$.
At any point on this frontier that minimizes $\IB_\beta^{\min} \equiv \min I(X;Z) - \beta I(Y;Z)$, the frontier will have slope $\beta^{-1}$ if it is differentiable. If the frontier is also concave (has negative second derivative), then this slope $\beta^{-1}$ will take its maximum $\beta_0^{-1}$ at the origin, which implies
$\IB_\beta$-Learnability for $\beta > \beta_0$, so that the threshold for $\IB_\beta$-Learnability is simply the inverse slope of the frontier at the origin.
More generally, as long as the Pareto frontier is differentiable, the threshold for IB$_\beta$-learnability is the inverse of its maximum slope. Indeed, Theorem \ref{thm:suff_2} and Theorem \ref{thm:suff_3} give lower bounds of the slope of the Pareto frontier at the origin.

\paragraph{IB-Learnability, hypercontractivity, and maximum correlation.}

IB-Learnability and its sufficient conditions we provide harbor a deep connection with hypercontractivity and maximum correlation:

\begin{align}
\frac{1}{\beta_0} &= \xi(X;Y)=\eta_\text{KL} \ge \sup_{h(x)}\frac{1}{\beta_0[h(x)]} = \rho_m^2(X;Y) \label{eq:relations_all}
\end{align}

which we prove in Appendix~\ref{app:maximum_corr}.
Here $\rho_m(X;Y)\equiv\max_{f,g} \mathbb{E}[f(X)g(Y)]$ s.t. $\mathbb{E}[f(X)]=\mathbb{E}[g(Y)]=0$ and $\mathbb{E}[f^2(X)]=\mathbb{E}[g^2(Y)]=1$ is the \textit{maximum correlation}~\citep{hirschfeld1935connection,gebelein1941statistische}, $\xi(X;Y)\equiv\sup_{Z-X-Y}\frac{I(Y;Z)}{I(X;Z)}$ is the \textit{hypercontractivity coefficient}, and $\eta_\text{KL}(p(y|x),p(x))\equiv\sup_{r(x)\neq p(x)}\frac{\mathbb{D}_\text{KL}(r(y)||p(y))}{\mathbb{D}_\text{KL}(r(x)||p(x))}$ is the \textit{contraction coefficient}.
Our proof relies on \citet{anantharam2013maximal}'s proof $\xi(X;Y)=\eta_\text{KL}$.
Our work reveals the deep relationship between IB-Learnability and these earlier concepts and provides additional insights about what aspects of a dataset give rise to high maximum correlation and hypercontractivity: the most confident, typical, imbalanced subset of $(X,Y)$.

\section{ESTIMATING THE IB-LEARNABILITY CONDITION}
\label{sec:estimate}

Theorem \ref{thm:suff_3} not only reveals the relationship between the learnability threshold for $\beta$ and the least noisy region of $P(Y|X)$, but also provides a way to practically estimate $\beta_0$, both in the general classification case, and in more structured settings.

\subsection{Estimation Algorithm}

Based on Theorem \ref{thm:suff_3}, for general classification tasks we suggest Algorithm \ref{alg:estimating_beta} to empirically estimate an upper-bound $\tilde{\beta}_{0} \ge \beta_0$, as well as discovering the conspicuous subset that determines $\beta_0$.

We approximate the probability of each example $p(x_i)$ by its empirical probability, $\hat{p}(x_i)$.
E.g., for MNIST, $p(x_i) = \frac{1}{N}$, where $N$ is the number of examples in the dataset.
The algorithm starts by first learning a maximum likelihood model of $p_\theta(y|x)$, using e.g. feed-forward neural networks.
It then constructs a matrix $P_{y|x}$ and a vector $p_y$ to store the estimated $p(y|x)$ and $p(y)$ for all the examples in the dataset.
To find the subset $\Omega$ such that the $\tilde{\beta}_0$ is as small as possible, by previous analysis we want to find a \emph{conspicuous} subset such that its $p(y|x)$ is large for a certain class $j$ (to make the denominator of Eq.~(\ref{eq:suff_3}) large), and containing as many elements as possible (to make the numerator small).

We suggest the following heuristics to discover such a conspicuous subset. For each class $j$,
we sort the rows of $(P_{y|x})$ according to its probability for the pivot class $j$ by decreasing order, and then perform a search over $i_\text{left}, i_\text{right}$ for $\Omega=\{i_\text{left},i_\text{left}+1,...,i_\text{right}\}$.
Since $\tilde{\beta}_0$ is large when $\Omega$ contains too few or too many elements, the minimum of $\tilde{\beta}_0^{(j)}$ for class $j$ will typically be reached with some intermediate-sized subset, and we can use binary search or other discrete search algorithm for the optimization.
The algorithm stops when $\tilde{\beta}_0^{(j)}$ does not improve by tolerance $\varepsilon$. The algorithm then returns the $\tilde{\beta}_0$ as the minimum over all the classes $\tilde{\beta}_0^{(1)},...\tilde{\beta}_0^{(N)}$, as well as the conspicuous subset that determines this $\tilde{\beta}_0$.

\begin{algorithm}[t]
 \caption{\textbf{Estimating the upper bound for $\beta_0$ and identifying the conspicuous subset}}
\label{alg:estimating_beta}
\begin{algorithmic}
\STATE {\bfseries Require}: Dataset $\D=\{(x_i,y_i)\},i=1,2,...N$. The number of classes is $C$.
\STATE {\bfseries Require} $\varepsilon$: tolerance for estimating $\beta_0$
\STATE 1: Learn a maximum likelihood model $p_\theta(y|x)$ using
\STATE \ \ \ \ the dataset $\D$.
\STATE 2: Construct matrix $(P_{y|x})$ such that 
\STATE \ \ \ \ $(P_{y|x})_{ij}=p_\theta(y=j|x=x_i)$.
\STATE 3: Construct vector $p_y=(p_{y1},..,p_{yC})$ such that 
\STATE \ \ \ \ $p_{yj}=\frac{1}{N}\sum_{i=1}^N(P_{y|x})_{ij}$.
\STATE 4: \textbf{for} $j$ \textbf{in} $\{1,2,...C\}$:
\STATE 5: \ \ \ \ $P_{y|x}^{(\text{sort} j)}\gets$Sort the rows of $P_{y|x}$ in decreasing 
\STATE \ \ \ \ \ \ \ \ values of $(P_{y|x})_{ij}$.
\STATE 6: \ \ \ \ $\tilde{\beta}_0^{(j)},\Omega^{(j)}\gets$Search $i_\text{left}$, $i_\text{right}$ until $\tilde{\beta}_0^{(j)}=$ 
\STATE \ \ \ \ \ \ \ \  $\textbf{Get}\boldsymbol{\beta}(P_{y|x},p_y,\Omega)$ is minimal with tolerance $\varepsilon$,
\STATE \ \ \ \ \ \ \ \ where $\Omega=\{i_\text{left},i_\text{left}+1,...i_\text{right}\}$.
\STATE 7: \textbf{end for}
\STATE 8: $j^*\gets\argmin_j\{\tilde{\beta}_0^{(j)}\}, j=1,2,...N$.
\STATE 9: $\tilde{\beta}_0\gets\tilde{\beta}_0^{(j^*)}$.
\STATE 10: $P_{y|x}^{(\tilde{\beta}_0)}\gets$ the rows of $P_{y|x}^{(\text{sort} j^*)}$ indexed by $\Omega^{(j^*)}$.
\STATE 11: \textbf{return} $\tilde{\beta}_0, P_{y|x}^{(\tilde{\beta}_0)}$
\STATE
\STATE \textbf{subroutine Get$\boldsymbol{\beta}$}($P_{y|x}, p_y, \Omega$):
\STATE s1: $N\gets $ number of rows of $P_{y|x}$.
\STATE s2: $C\gets $ number of columns of $P_{y|x}$.
\STATE s3: $n\gets$ number of elements of $\Omega$.
\STATE s4: $(p_{y|\Omega})_j\gets\frac{1}{n}\sum_{i\in\Omega}(P_{y|x})_{ij}$, $j=1,2,...,C$.
\STATE s5: $\tilde{\beta}_0\gets \frac{\frac{N}{n} - 1}{\sum_{j} \big[ \frac{(p_{y|\Omega_x})_j^2}{p_{yj}} - 1 \big]}$
\STATE s6: \textbf{return} $\tilde{\beta}_0$
\end{algorithmic}
\end{algorithm}

After estimating $\tilde{\beta}_0$, we can then use it for learning with IB, either directly, or as an anchor for a region where we can perform a much smaller sweep than we otherwise would have.
This may be particularly important for very noisy datasets, where $\beta_0$ can be very large.

\subsection{Special Cases for Estimating $\beta_0$}

Theorem \ref{thm:suff_3} may still be challenging to estimate, due to the difficulty of making accurate estimates of $p(\Omega_x)$ and searching over $\Omega_x\subset\X$.
However, if the learning problem is more structured, we may be able to obtain a simpler formula for the sufficient condition.

\paragraph{Class-conditional label noise.}
Classification with noisy labels is a common practical scenario.
An important noise model is that the labels are randomly flipped with some hidden class-conditional probabilities and we only observe the corrupted labels.
This problem has been studied extensively \citep{angluin1988learning,natarajan2013learning,liu2016classification,xiao2015learning,northcutt2017learning}.
If IB is applied to this scenario, how large $\beta$ do we need?
The following corollary provides a simple formula.

\begin{corollary}
\label{corollary:suff_3_class_conditional}
Suppose that the true class labels are $y^*$, and the input space belonging to each $y^*$ has no overlap.
We only observe the corrupted labels $y$ with class-conditional noise $p(y|x,y^*) = p(y|y^*)$, and $Y$ is not independent of $X$.
We have that a sufficient condition for $\IB_\beta$-Learnability is:
\begin{equation}
\begin{aligned}
\label{eq:suff_3_class_conditional}
\beta > \inf_{y^*} \frac{\frac{1}{p(y^*)} - 1}{\sum_y \frac{p(y|y^*)^2}{p(y)} - 1}
\end{aligned}
\end{equation}
\end{corollary}

We see that under class-conditional noise, the sufficient condition reduces to a discrete formula which only depends on the noise rates $p(y|y^*)$ and the true class probability $p(y^*)$, which can be accurately estimated via e.g. \citet{northcutt2017learning}.
Additionally, if we know that the noise is class-conditional, but the observed $\beta_0$ is greater than the R.H.S. of Eq.~(\ref{eq:suff_3_class_conditional}), we can deduce that there is overlap between the true classes.
The proof of Corollary \ref{corollary:suff_3_class_conditional} is provided in Appendix \ref{app:corollaries}.

\paragraph{Deterministic relationships.}
Theorem \ref{thm:suff_3} also reveals that $\beta_0$ relates closely to whether $Y$ is a deterministic function of $X$, as shown by Corollary \ref{corollary:suff_3_2}:
\begin{corollary}
\label{corollary:suff_3_2}
Assume that $Y$ contains at least one value $y$ such that its probability $p(y)>0$. If $Y$ is a deterministic function of $X$ and not independent of $X$, then a sufficient condition for $\IB_\beta$-Learnability is $\beta > 1$.
\end{corollary}

The assumption in the corollary \ref{corollary:suff_3_2} is satisfied by classification, and certain regression problems.
Combined with the necessary condition $\beta>1$ for any dataset $(X,Y)$ to be $\IB_\beta$-learnable (Section \ref{sec:learnability}), we have that under the assumption, if $Y$ is a deterministic function of $X$, then a necessary and sufficient condition for $\IB_\beta$-learnability is $\beta>1$; i.e., its $\beta_0$ is 1.
The proof of Corollary \ref{corollary:suff_3_2} is provided in Appendix \ref{app:corollaries}.

Therefore, in practice, if we find that $\beta_0 > 1$, we may infer that $Y$ is not a deterministic function of $X$.
For a classification task, we may infer that either some classes have overlap, or the labels are noisy.
However, recall that finite models may add effective class overlap if they have insufficient capacity for the learning task, as mentioned in Section~\ref{sec:suff_conditions}.
This may translate into a higher observed $\beta_0$, even when learning deterministic functions.

\begin{figure}[t]
\begin{center}
\includegraphics[width=1\columnwidth]{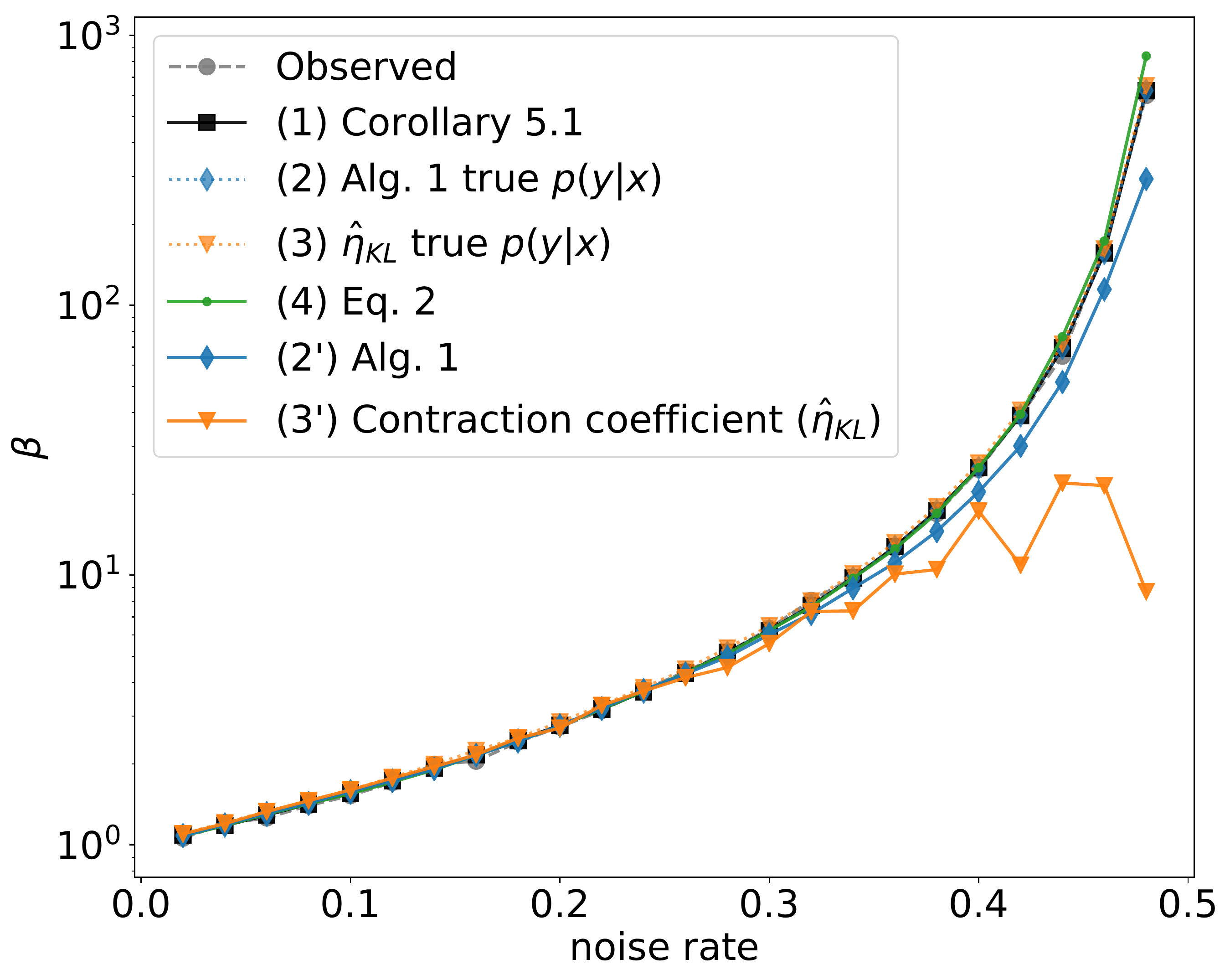}
\end{center}
\caption{Predicted vs. experimentally identified $\beta_0$, for mixture of Gaussians with varying class-conditional noise rates.}
\label{fig:gauss_noise_beta}
\end{figure}

\section{EXPERIMENTS}
\label{sec:experiments}

To test how the theoretical conditions for $\IB_\beta$-learnability match with experiment, we apply them to synthetic data with varying noise rates and class overlap, MNIST binary classification with varying noise rates, and CIFAR10 classification, comparing with the $\beta_0$ found experimentally.
We also compare with the algorithm in \citet{kim2017discovering} for estimating the hypercontractivity coefficient (=$1/\beta_0$) via the contraction coefficient $\eta_{\text{KL}}$.
Experiment details are in Section~\ref{app:experiment}.

\subsection{Synthetic Dataset Experiments}
\label{sec:synthetic_exp}

We construct a set of datasets from 2D mixtures of 2 Gaussians as $X$ and the identity of the mixture component as $Y$.
We simulate two practical scenarios with these datasets: \textbf{(1)} noisy labels with class-conditional noise, and \textbf{(2)} class overlap.
For (1), we vary the class-conditional noise rates.
For (2), we vary class overlap by tuning the distance between the Gaussians.
For each experiment, we sweep $\beta$ with exponential steps, and observe $I(X;Z)$ and $I(Y;Z)$.
We then compare the empirical $\beta_0$ indicated by the onset of above-zero $I(X;Z)$ with predicted values for $\beta_0$.

\begin{figure}[t!]
\begin{center}
\includegraphics[width=1\columnwidth]{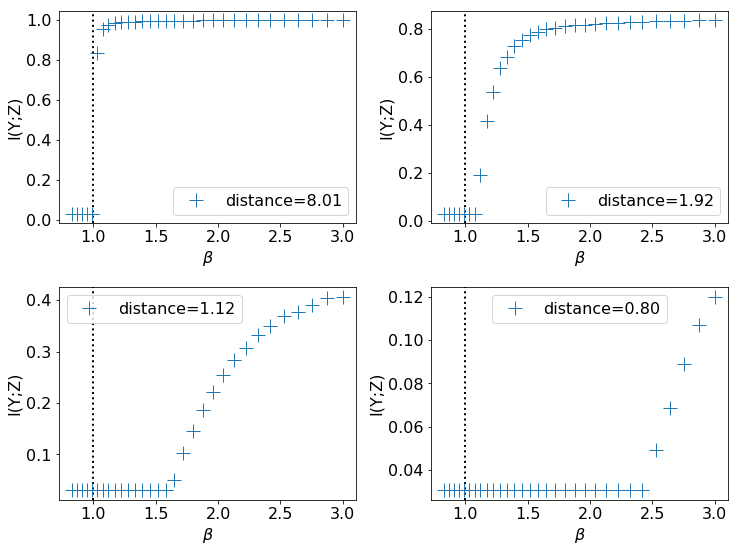}
\end{center}
\caption{
$I(Y;Z)$ vs. $\beta$, for mixture of Gaussian datasets with different distances between the two mixture components.
The vertical lines are $\beta_{0,\text{predicted}}$ computed by the R.H.S. of Eq.~(\ref{eq:suff_3_class_conditional}).
As Eq.~(\ref{eq:suff_3_class_conditional}) does not make predictions w.r.t. class overlap, the vertical lines are always just above $\beta_{0,\text{predicted}} = 1$.
However, as expected, decreasing the distance between the classes in $X$ space also increases the true $\beta_0$.
}
\label{fig:gauss_overlap_beta_indi}
\end{figure}

\paragraph{Classification with class-conditional noise.}
In this experiment, we have a mixture of Gaussian distribution with 2 components, each of which is a 2D Gaussian with diagonal covariance matrix $\Sigma=\text{diag}(0.25, 0.25)$.
The two components have distance 16 (hence virtually no overlap) and equal mixture weight.
For each $x$, the label $y \in \{0,1\}$ is the identity of which component it belongs to.
We create multiple datasets by randomly flipping the labels $y$ with a certain noise rate $\rho = P(y=0|y^*=1) = P(y=1|y^*=0)$.
For each dataset, we train VIB models across a range of $\beta$, and observe the onset of learning via random $I(X;Z)$ (Observed).
To test how different methods perform in estimating $\beta_0$, we apply the following methods:
\textbf{(1)} Corollary 5.1, since this is classification with class-conditional noise, and the two true classes have virtually no overlap;
\textbf{(2)} Alg.~\ref{alg:estimating_beta} with true $p(y|x)$;
\textbf{(3)} The algorithm in \citet{kim2017discovering} that estimates $\hat{\eta}_{\KL}$, provided with true $p(y|x)$;
\textbf{(4)} $\beta_0[h(x)]$ in Eq.~(\ref{eq:suff_2});
\textbf{(2$'$)} Alg.~\ref{alg:estimating_beta} with $p(y|x)$ estimated by a neural net;
\textbf{(3$'$)} $\hat{\eta}_{\KL}$ with the same $p(y|x)$ as in (2$'$).
The results are shown in Fig.~\ref{fig:gauss_noise_beta} and in Appendix~\ref{app:gauss_noise_beta}.

From Fig.~\ref{fig:gauss_noise_beta} we see the following. 
\textbf{(A)} When using the true $p(y|x)$, both Alg.~\ref{alg:estimating_beta} and $\hat{\eta}_\text{KL}$ generally upper bound the empirical $\beta_0$, and Alg.~\ref{alg:estimating_beta} is generally tighter.
\textbf{(B)} When using the true $p(y|x)$, Alg.~\ref{alg:estimating_beta} and Corollary \ref{corollary:suff_3_class_conditional} give the same result.
\textbf{(C)} Comparing Alg.~\ref{alg:estimating_beta} and $\hat{\eta}_{\KL}$ both of which use the same empirically estimated $p(y|x)$, both approaches provide good estimation in the low-noise region; however, in the high-noise region, Alg.~\ref{alg:estimating_beta} gives more precise values than $\hat{\eta}_{\KL}$, indicating that Alg.~\ref{alg:estimating_beta} is more robust to the estimation error of $p(y|x)$. 
\textbf{(D)} Eq.~(\ref{eq:suff_2}) empirically upper bounds the experimentally observed $\beta_0$, and gives almost the same result as theoretical estimation in Corollary \ref{corollary:suff_3_class_conditional} and Alg.~\ref{alg:estimating_beta} with the true $p(y|x)$.
In the classification setting, this approach doesn't require any learned estimate of $p(y|x)$, as we can directly use the empirical $p(y)$ and $p(x|y)$ from SGD mini-batches.

This experiment also shows that for dataset where the signal-to-noise is small, $\beta_0$ can be very high.
Instead of blindly sweeping $\beta$, our result can provide guidance for setting $\beta$ so learning can happen.

\begin{figure}[t!]
\begin{center}
\includegraphics[width=1\columnwidth]{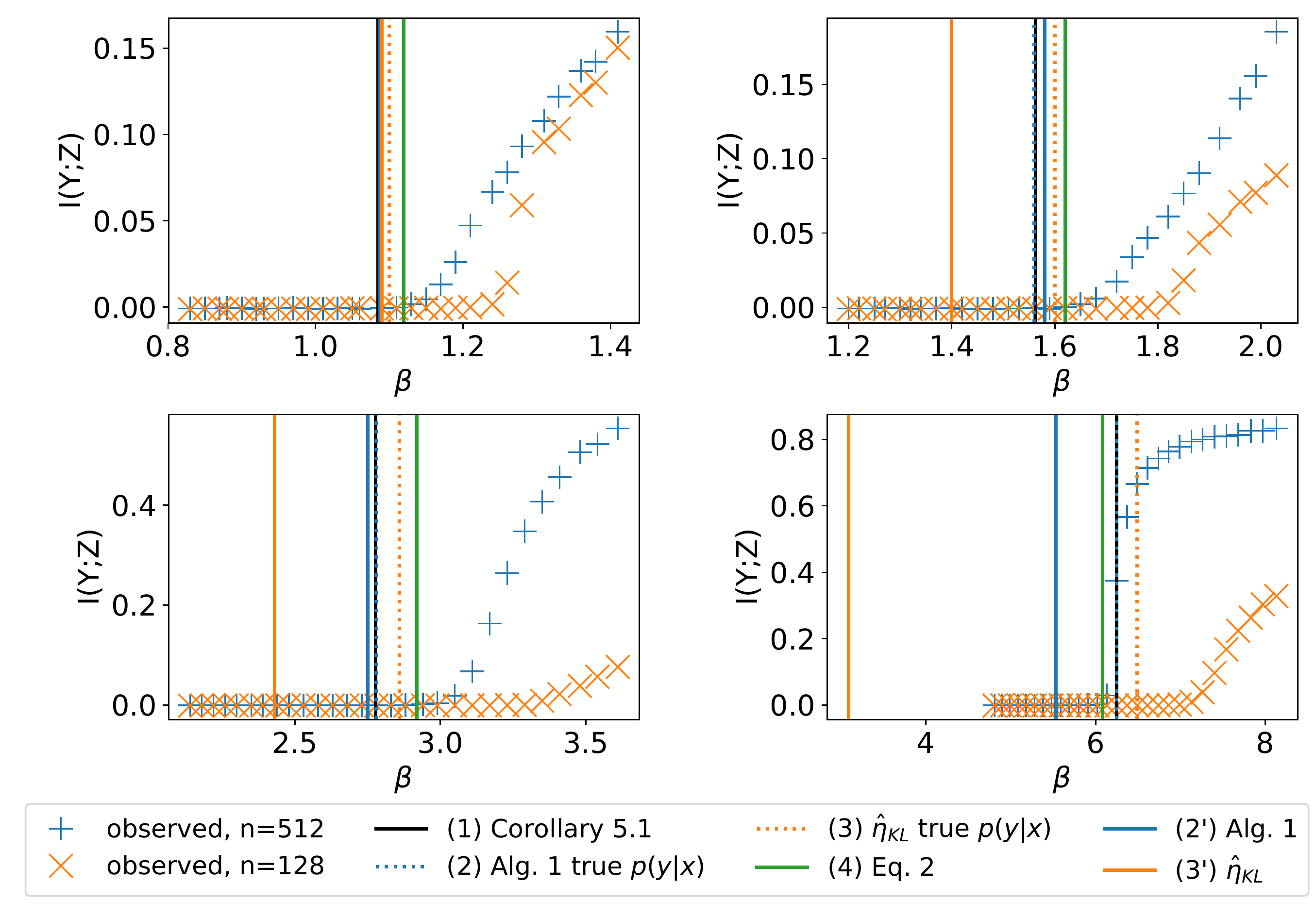}
\end{center}
\caption{
$I(Y;Z)$ vs. $\beta$ for the MNIST binary classification with different hidden units per layer $n$ and noise rates $\rho$: (upper left) $\rho=0.02$, (upper right) $\rho=0.1$, (lower left) $\rho=0.2$, (lower right) $\rho=0.3$.
The vertical lines are $\beta_{0}$ estimated by different methods.
$n=128$ has insufficient capacity for the problem, so its observed learnability onset is pushed higher, similar to the class overlap case.
}
\label{fig:mnist_noise_beta_indi}
\end{figure}

\paragraph{Classification with class overlap.}
\label{sec:exp_overlap}

In this experiment, we test how different amounts of overlap among classes influence $\beta_0$.
We use the mixture of Gaussians with two components, each of which is a 2D Gaussian with diagonal covariance matrix $\Sigma=\text{diag}(0.25,0.25)$.
The two components have weights 0.6 and 0.4.
We vary the distance between the Gaussians from 8.0 down to 0.8 and observe the $\beta_{0,exp}$.
Since we don't add noise to the labels, if there were no overlap and a deterministic map from $X$ to $Y$, we would have $\beta_0=1$ by Corollary \ref{corollary:suff_3_2}.
The more overlap between the two classes, the more uncertain $Y$ is given $X$.
By Eq.~\ref{eq:suff_3} we expect $\beta_0$ to be larger, which is corroborated in Fig.~\ref{fig:gauss_overlap_beta_indi}.

\subsection{MNIST Experiments}

We perform binary classification with digits 0 and 1, and as before, add class-conditional noise to the labels with varying noise rates $\rho$.
To explore how the model capacity influences the onset of learning, for each dataset we train two sets of VIB models differing only by the number of neurons in their hidden layers of the encoder: one with $n=512$ neurons, the other with $n=128$ neurons.
As we describe in Section \ref{sec:suff_conditions}, insufficient capacity will result in more uncertainty of $Y$ given $X$ from the point of view of the model, so we expect the observed $\beta_{0}$ for the $n=128$ model to be larger.
This result is confirmed by the experiment (Fig. \ref{fig:mnist_noise_beta_indi}). Also, in Fig. \ref{fig:mnist_noise_beta_indi} 
we plot $\beta_0$ given by different estimation methods. We see that the observations (A), (B), (C) and (D) in Section \ref{sec:synthetic_exp} still hold.

\subsection{MNIST Experiments using Equation \ref{eq:suff_2}}
\label{sec:mnist_eq2}

To see what IB learns at its onset of learning for the full MNIST dataset, we optimize Eq. (\ref{eq:suff_2}) w.r.t. the full MNIST dataset, and visualize the clustering of digits by $h(x)$.
Eq.~(\ref{eq:suff_2}) can be optimized using SGD using any differentiable parameterized mapping $h(x) : \mathcal{X} \rightarrow \mathbbm{R}$.
In this case, we chose to parameterize $h(x)$ with a PixelCNN++ architecture~\citep{pixelcnn,pxpp}, as PixelCNN++ is a powerful autoregressive model for images that gives a scalar output (normally interpreted as $\log\,p(x)$).
Eq.~(\ref{eq:suff_2}) should generally give two clusters in the output space, as discussed in Section~\ref{sec:suff_conditions}.
In this setup, smaller values of $h(x)$ correspond to the subset of the data that is easiest to learn.
Fig.~\ref{fig:mnist_hx_hist} shows two strongly separated clusters, as well as the threshold we choose to divide them.
Fig.~\ref{fig:mnist_sort_by_hx} shows the first 5,776 MNIST training examples as sorted by our learned $h(x)$, with the examples above the threshold highlighted in red.
We can clearly see that our learned $h(x)$ has separated the ``easy'' one (1) digits from the rest of the MNIST training set.

\begin{figure}[t!]
\begin{center}
\includegraphics[width=\columnwidth]{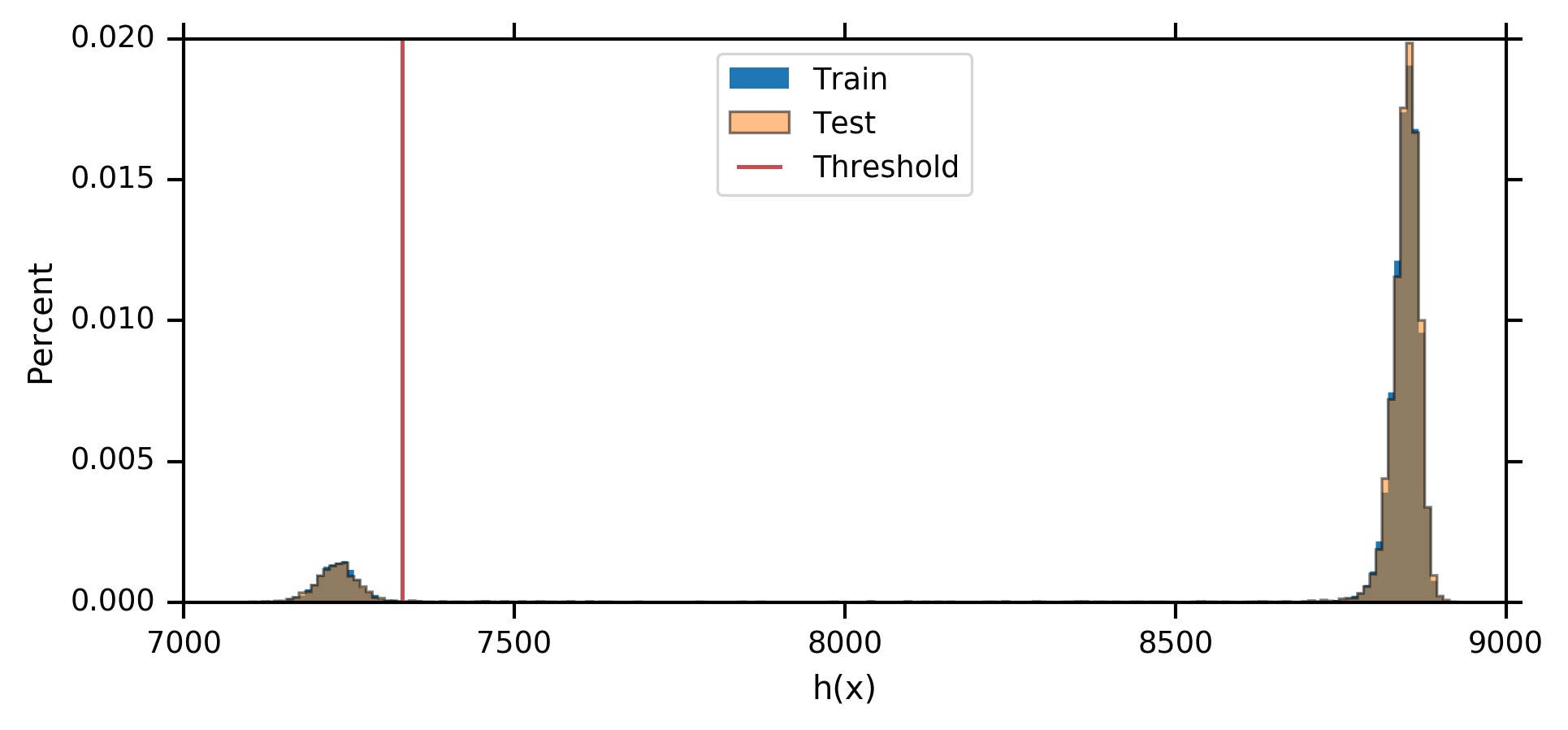}
\end{center}
\caption{
Histograms of the full MNIST training and validation sets according to $h(X)$.
Note that both are bimodal, and the histograms are indistinguishable.
In both cases, $h(x)$ has learned to separate most of the ones into the smaller mode, but difficult ones are in the wide valley between the two modes.
See Figure~\ref{fig:mnist_sort_by_hx} for all of the training images to the left of the red threshold line, as well as the first few images to the right of the threshold.
}
\label{fig:mnist_hx_hist}
\end{figure}

\begin{figure}[t!]
\begin{center}
\includegraphics[width=0.9\columnwidth]{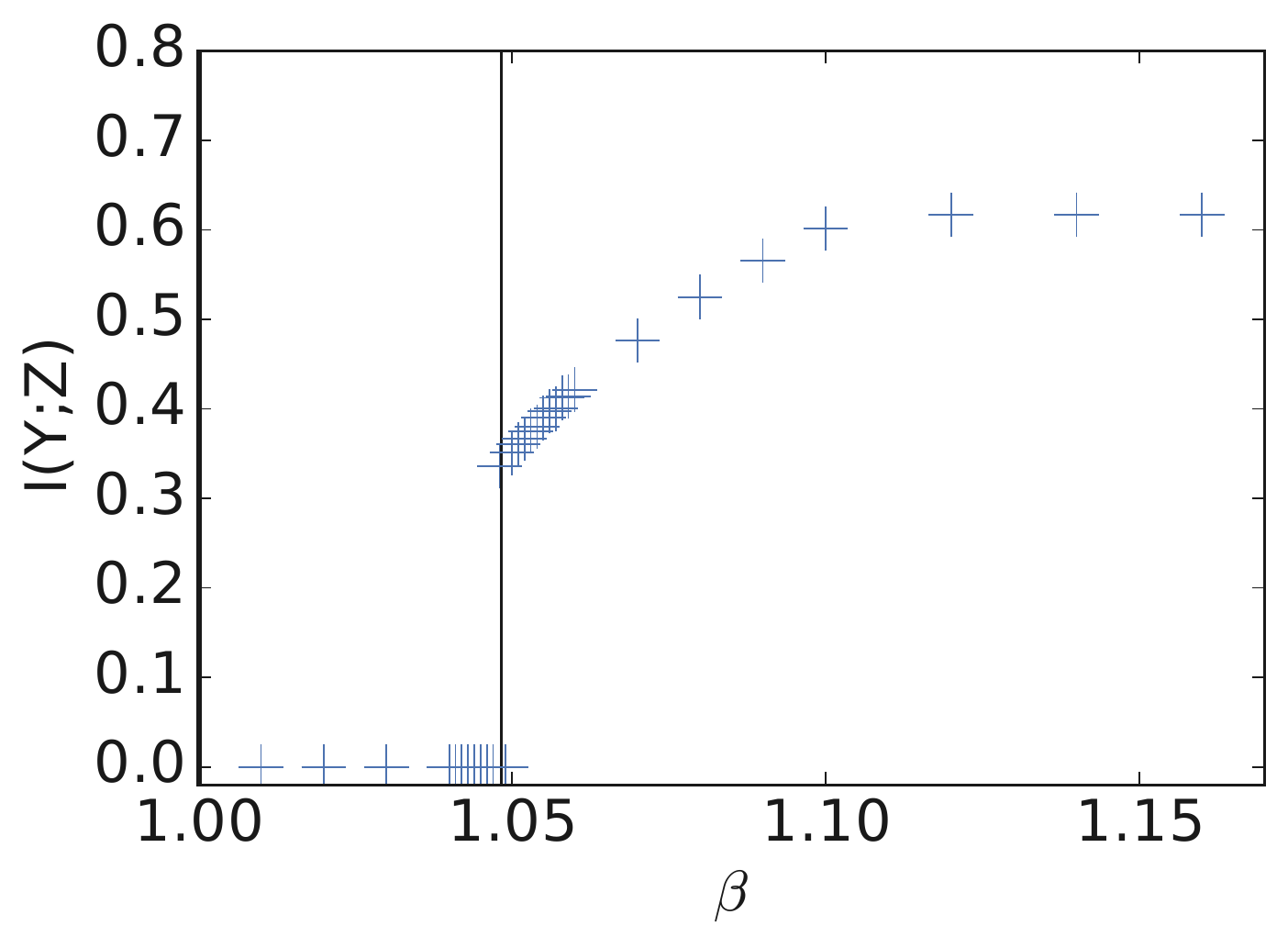}
\end{center}
\vskip -0.12in
\caption{
Plot of $I(Y;Z)$ vs $\beta$ for CIFAR10 training set with 20\% label noise.
Each blue cross corresponds to a fully-converged model starting with independent initialization.
The vertical black line corresponds to the predicted $\beta_0=1.0483$ using Alg.~\ref{alg:estimating_beta}.
The empirical $\beta_0=1.048$.
}
\label{fig:cifar10_experiments}
\end{figure}

\begin{figure*}[p]
\begin{center}
\includegraphics[width=2\columnwidth]{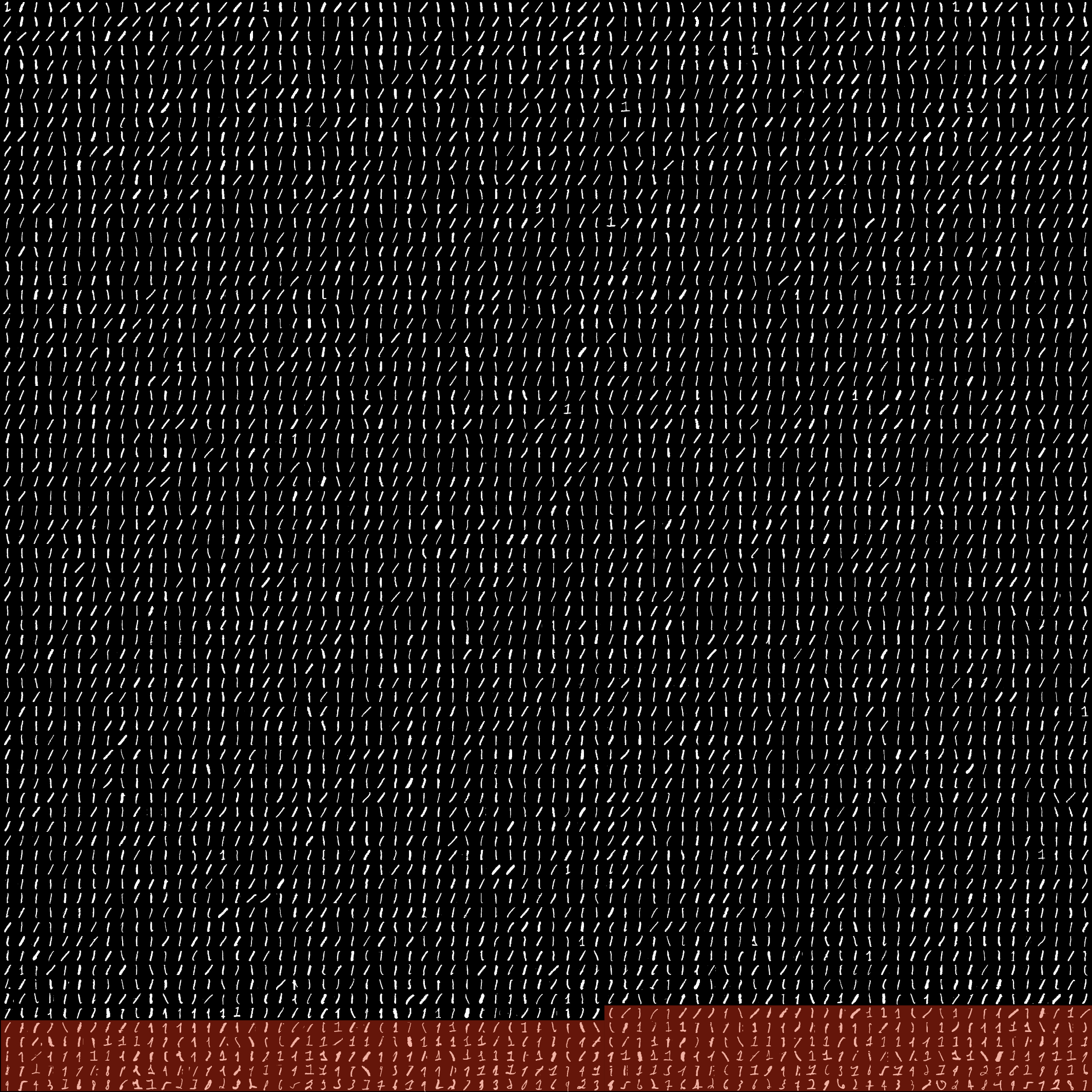}
\end{center}
\caption{
The first 5776 MNIST training set digits when sorted by $h(x)$.
The digits highlighted in red are above the threshold drawn in Figure~\ref{fig:mnist_hx_hist}.
}
\label{fig:mnist_sort_by_hx}
\end{figure*}

\subsection{CIFAR10 Forgetting Experiments}

For CIFAR10~\citep{cifar}, we study how \textit{forgetting} varies with $\beta$.
In other words, given a VIB model trained at some high $\beta_2$, if we anneal it down to some much lower $\beta_1$, what $I(Y;Z)$ does the model converge to?
Using Alg.~\ref{alg:estimating_beta}, we estimated $\beta_0 = 1.0483$ on a version of CIFAR10 with 20\% label noise, where the $P_{y|x}$ is estimated by maximum likelihood training with the same encoder and classifier architectures as used for VIB. 
For the VIB models, the lowest $\beta$ with performance above chance was $\beta = 1.048$, a very tight match with the estimate from Alg.~\ref{alg:estimating_beta}.
See Appendix~\ref{app:cifar_details} for details.

\section{CONCLUSION}

In this paper, we have presented theoretical results for predicting the onset of learning, and have shown that it is determined by the conspicuous subset of the training examples.
We gave a practical algorithm for predicting the transition as well as discovering this subset, and showed that those predictions are accurate, even in cases of extreme label noise.
We believe these results will provide theoretical and practical guidance for choosing $\beta$ in the IB framework for balancing prediction and compression.
Our work also raises other questions, such as whether there are other phase transitions in learnability that might be identified.
We hope to address some of those questions in future work.

\subsubsection*{Acknowledgements}

Tailin Wu's work was supported by the The Casey and Family Foundation, the Foundational Questions Institute and the Rothberg Family Fund for Cognitive Science. He thanks the Center for Brains, Minds, and Machines (CBMM) for hospitality.

\clearpage

\bibliography{reference}
\bibliographystyle{plainnat}

\input{Appendix}

\end{document}

%% file: Appendix.tex
\newpage
\onecolumn
\appendix
\begin{center}
\begin{huge}
\textbf{Appendix}
\end{huge}
\end{center}

The structure of the Appendix is as follows.
In Section~\ref{app:variation_background}, we provide preliminaries for the first-order and second-order variations on functionals. We prove Theorem \ref{thm:homo_learnability} and Theorem \ref{thm:beta_monotonic} in Section \ref{app:homo_learnability} and \ref{app:beta_monotonic}, respectively.
In Section~\ref{app:suff_1}, we prove Theorem \ref{thm:suff_1}, the sufficient condition 1 for IB-Learnability.
In Section~\ref{app:first_second_variations}, we calculate the first and second variations of $\IB_\beta[p(z|x)]$ at the trivial representation $p(z|x)=p(z)$, which is used in proving the Sufficient Condition 2 for $\IB_\beta$-learnability (Section~\ref{app:suff_2}). In Appendix \ref{app:what_first_learns}, we prove Eq. (\ref{eq:what_first_learns}) at the onset of learning.
After these preparations, we prove the key result of this paper, Theorem~\ref{thm:suff_3}, in Section~\ref{app:suff_3}.
Then two important corollaries \ref{corollary:suff_3_class_conditional}, \ref{corollary:suff_3_2} are proved in Section~\ref{app:corollaries}.
In Section~\ref{app:maximum_corr} we explore the deep relation between $\beta_0$, $\beta_0[h(x)]$, the hypercontractivity coefficient, contraction coefficient and maximum correlation.
Finally in Section~\ref{app:experiment}, we provide details for the experiments.

\section{Preliminaries: first-order and second-order variations}
\label{app:variation_background}

Let functional $F[f(x)]$ be defined on some normed linear space $\mathscr{R}$. Let us add a perturbative function $\epsilon\cdot h(x)$ to $f(x)$, and now the functional 
$F[f(x)+\epsilon\cdot h(x)]$ can be expanded as

\begin{equation*}
\begin{aligned}
\Delta F[f(x)]&=F[f(x)+\epsilon \cdot h(x)]-F[f(x)]\\
&=\varphi_1[f(x)]+\varphi_2[f(x)]+\O(\epsilon^3 ||h||^2)
\end{aligned}
\end{equation*}

where $||h||$ denotes the norm of $h$, $\varphi_1[f(x)]=\epsilon \frac{d F[f(x)]}{d\epsilon}$ is a linear functional of $\epsilon \cdot h(x)$, and is called the \emph{first-order variation}, denoted as $\delta F[f(x)]$. $\varphi_2[f(x)]=\frac{1}{2}\epsilon^2 \frac{d^2 F[f(x)]}{d\epsilon^2}$ is a quadratic functional of $\epsilon \cdot h(x)$, and is called the \emph{second-order variation}, denoted as $\delta^2 F[f(x)]$.

If $\delta F[f(x)]=0$, we call $f(x)$ a stationary solution for the functional $F[\cdot]$.

If $\Delta F[f(x)]\geq0$ for all $h(x)$ such that $f(x)+\epsilon\cdot h(x)$ is at the neighborhood of $f(x)$, we call $f(x)$ a (local) minimum of $F[\cdot]$.

\section{Proof of Theorem \ref{thm:homo_learnability}}
\label{app:homo_learnability}
\begin{proof}
If $(X,Y)$ is $\IB_\beta$-learnable, then there exists $Z$ given by some $p_1(z|x)$ such that $\IB_\beta(X,Y;Z)<\IB(X,Y;Z_{trivial})=0$, where $Z_{trivial}$ satisfies $p(z|x)=p(z)$. Since $X'=g(X)$ is a invertible map (if $X$ is continuous variable, $g$ is additionally required to be continuous), and mutual information is invariant under such an invertible map (\cite{kraskov2004estimating}), we have that $\IB_\beta(X',Y;Z)=I(X';Z)-\beta I(Y;Z)=I(X;Z)-\beta I(Y;Z)=\IB_\beta(X,Y;Z)<0=\IB(X',Y;Z_{trivial})$, so $(X',Y)$ is $\IB_\beta$-learnable. On the other hand, if $(X,Y)$ is not $\IB_\beta$-learnable, then $\forall Z$, we have $\IB_\beta(X,Y;Z)\ge\IB(X,Y;Z_{trivial})=0$. Again using mutual information's invariance under $g$, we have for all $Z$, $\IB_\beta(X',Y;Z)=\IB_\beta(X,Y;Z)\ge\IB(X,Y;Z_{trivial})=0$, leading to that $(X',Y)$ is not $\IB_\beta$-learnable. Therefore, we have that $(X,Y)$ and $(X',Y)$ have the same $\IB_\beta$-learnability. 

\end{proof}

\section{Proof of Theorem \ref{thm:beta_monotonic}}
\label{app:beta_monotonic}

\begin{proof}
At the trivial representation $p(z|x)=p(z)$, we have $I(X;Z)=0$, and $I(Y;Z)=0$ due to the Markov chain, so $\IB_\beta(X,Y;Z)\rvert_{p(z|x)=p(z)} = 0$ for any $\beta$.
Since $(X,Y)$ is $\IB_{\beta_1}$-learnable, there exists a $Z$ given by a $p_1(z|x)$ such that $\IB_{\beta_1}(X,Y;Z)\rvert_{p_1(z|x)} < 0$.
Since $\beta_2 > \beta_1$, and $I(Y;Z) \geq 0$, we have $\IB_{\beta_2}(X,Y;Z)\rvert_{p_1(z|x)} \leq \IB_{\beta_1}(X,Y;Z)\rvert_{p_1(z|x)} < 0 = \IB_{\beta_2}(X,Y;Z)\rvert_{p(z|x)=p(z)}$.
Therefore, $(X,Y)$ is $\IB_{\beta_2}$-learnable.
\end{proof}

\section{Proof of Theorem \ref{thm:suff_1}}
\label{app:suff_1}

\begin{proof}
To prove Theorem \ref{thm:suff_1}, we use the Theorem 1 of Chapter 5 of \citet{gelfand2000calculus} which gives a necessary condition for $F[f(x)]$ to have a minimum at $f_0(x)$.
Adapting to our notation, we have:

\begin{theorem}[\cite{gelfand2000calculus}]
\label{thm:necessary_minimum}
A necessary condition for the functional $F[f(x)]$ to have a minimum at $f(x)=f_0(x)$ is that for $f(x)=f_0(x)$ and all admissible $\epsilon\cdot h(x)$, 
$$\delta^2 F[f(x)]\geq0$$.
\end{theorem}

Applying to our functional $\IB_\beta[p(z|x)]$, an immediate result of Theorem \ref{thm:necessary_minimum} is that, if at $p(z|x)=p(z)$, there exists an $\epsilon \cdot h(z|x)$ such that $\delta^2 \IB_\beta[p(z|x)]<0$, then $p(z|x)=p(z)$ is not a minimum for $\IB_\beta[p(z|x)]$. Using the definition of $\IB_\beta$ learnability, we have that $(X,Y)$ is $\IB_\beta$-learnable.

\end{proof}

\section{First- and second-order variations of $IB_\beta[p(z|x)]$}
\label{app:first_second_variations}

In this section, we derive the first- and second-order variations of $\IB_\beta[p(z|x)]$, which are needed for proving Lemma \ref{lemma:stationary} and Theorem \ref{thm:suff_2}.

\begin{lemma}
\label{lemma:first_second_variation_IB}
\text{Using perturbative function $h(z|x)$, we have}

\begin{equation*}
\begin{aligned}
&\delta \IB_\beta[p(z|x)]=\int dx dz p(x) h(z|x)\emph{\log}\frac{p(z|x)}{p(z)}-\beta \int dx dy dz p(x,y)h(z|x)\emph{\log}\frac{p(z|y)}{p(z)}\\
&\delta^2 \IB_\beta[p(z|x)]=\\
&\frac{1}{2}\bigg[\int dxdz\frac{p(x)^2}{p(x,z)}h(z|x)^2-\beta\int dx dx' dy dz\frac{p(x,y)p(x',y)}{p(y,z)}h(z|x)h(z|x')+(\beta-1)\int dx dx' dz \frac{p(x)p(x')}{p(z)}h(z|x)h(z|x')\bigg]
\end{aligned}
\end{equation*}
\end{lemma}

\begin{proof}
Since $\IB_\beta[p(z|x)]=I(X;Z)-\beta I(Y;Z)$, let us calculate the first and second-order variation of $I(X;Z)$ and $I(Y;Z)$ w.r.t. $p(z|x)$, respectively. Through this derivation, we use $\epsilon h(z|x)$ as a perturbative function, for ease of deciding different orders of variations. We will finally absorb $\epsilon$ into $h(z|x)$.

Denote $I(X;Z)=F_1[p(z|x)]$. We have
\begin{equation*}
F_1[p(z|x)]=I(X;Z)=\int dx dz p(z|x)p(x)\log\frac{p(z|x)}{p(z)}
\end{equation*}

Since
$$p(z)=\int p(z|x)p(x)dx$$
We have
$$p(z)\rvert_{p(z|x)+\epsilon h(z|x)}=p(z)\rvert_{p(z|x)}+\epsilon\int h(z|x)p(x)dx$$

Expanding $F_1[p(z|x)+\epsilon h(z|x)]$ to the second order of $\epsilon$, we have

\begin{equation*}
\begin{aligned}
&F_1[p(z|x)+\epsilon h(z|x)]\\
&=\int dx dz p(x)[p(z|x)+\epsilon h(z|x)]\log\frac{p(z|x)+\epsilon h(z|x)}{p(z)+\epsilon\int h(z|x')p(x')dx'}\\
&=\int dx dz p(x)p(z|x)\left(1+\epsilon\frac{ h(z|x)}{p(z|x)}\right)\log\frac{p(z|x)\left(1+\epsilon\frac{h(z|x)}{p(z|x)}\right)}{p(z)\left(1+\epsilon\frac{\int h(z|x')p(x')dx'}{p(z)}\right)}\\
&=\int dx dz p(x)p(z|x)\left(1+\epsilon\frac{ h(z|x)}{p(z|x)}\right)\log\bigg[\frac{p(z|x)}{p(z)}\bigg(1+\epsilon\frac{h(z|x)}{p(z|x)}\bigg)\bigg(1-\epsilon\frac{\int h(z|x')p(x')dx'}{p(z)}\\
&+\epsilon^2\bigg(\frac{\int h(z|x')p(x')dx'}{p(z)}\bigg)^2\bigg)\bigg]+\O(\epsilon^3)\\
&=\int dx dz p(x)p(z|x)\left(1+\epsilon\frac{ h(z|x)}{p(z|x)}\right)\log\bigg[\frac{p(z|x)}{p(z)}\bigg(1+\epsilon\bigg(\frac{h(z|x)}{p(z|x)}-\frac{\int h(z|x')p(x')dx'}{p(z)}\bigg)\\
&+\epsilon^2\left(\frac{\int h(z|x')p(x')dx'}{p(z)}\right)^2-\epsilon^2\frac{h(z|x)}{p(z|x)}\frac{\int h(z|x')p(x')dx'}{p(z)}\bigg)\bigg]+\O(\epsilon^3)\\
&=\int dx dz p(x)p(z|x)\left(1+\epsilon\frac{ h(z|x)}{p(z|x)}\right)\bigg[\log\frac{p(z|x)}{p(z)}+\epsilon\bigg(\frac{h(z|x)}{p(z|x)}-\frac{\int h(z|x')p(x')dx'}{p(z)}\bigg)\\
&+\epsilon^2\left(\frac{\int h(z|x')p(x')dx'}{p(z)}\right)^2-\epsilon^2\frac{h(z|x)}{p(z|x)}\frac{\int h(z|x')p(x')dx'}{p(z)}-\frac{1}{2}\epsilon^2\bigg(\frac{h(z|x)}{p(z|x)}-\frac{\int h(z|x')p(x')dx'}{p(z)}\bigg)^2\bigg]+\O(\epsilon^3)\\
\end{aligned}
\end{equation*}

Collecting the first order terms of $\epsilon$, we have
\begin{equation*}
\begin{aligned}
&\delta F_1[p(z|x)]\\
&=\epsilon\int dx dz p(x)p(z|x)\bigg(\frac{h(z|x)}{p(z|x)}-\frac{\int h(z|x')p(x')dx'}{p(z)}\bigg)+\epsilon\int dx dz p(x)p(z|x)\frac{ h(z|x)}{p(z|x)}\log\frac{p(z|x)}{p(z)}\\
&=\epsilon\int dx dz p(x)h(z|x)-\epsilon\int dx' dz p(x')h(z|x')+\epsilon\int dx dz p(x) h(z|x)\log\frac{p(z|x)}{p(z)}\\
&=\epsilon\int dx dz p(x) h(z|x)\log\frac{p(z|x)}{p(z)}\\
\end{aligned}
\end{equation*}

Collecting the second order terms of $\epsilon^2$, we have

\begin{equation*}
\begin{aligned}
&\delta^2 F_1[p(z|x)]\\
&=\epsilon^2\int dx dz p(x)p(z|x)\bigg[\left(\frac{\int h(z|x')p(x')dx'}{p(z)}\right)^2-\frac{h(z|x)}{p(z|x)}\frac{\int h(z|x')p(x')dx'}{p(z)}-\frac{1}{2}\bigg(\frac{h(z|x)}{p(z|x)}-\frac{\int h(z|x')p(x')dx'}{p(z)}\bigg)^2\bigg]\\
&+\epsilon^2\int dx dz p(x)p(z|x)\frac{ h(z|x)}{p(z|x)}\bigg(\frac{h(z|x)}{p(z|x)}-\frac{\int h(z|x')p(x')dx'}{p(z)}\bigg)\\
&=\frac{\epsilon^2}{2}\int dxdz\frac{p(x)^2}{p(x,z)}h(z|x)^2-\frac{\epsilon^2}{2}\int dx dx' dz \frac{p(x)p(x')}{p(z)}h(z|x)h(z|x')
\end{aligned}
\end{equation*}

Now let us calculate the first and second-order variation of $F_2[p(z|x)]=I(Z;Y)$. We have
\begin{equation*}
F_2[p(z|x)]=I(Y;Z)=\int dy dz p(z|y)p(y)\log\frac{p(y,z)}{p(y)p(z)}=\int dx dy dz p(z|y)p(x,y)\log\frac{p(y,z)}{p(y)p(z)}
\end{equation*}
Using the Markov chain $Z\gets X\leftrightarrow Y$, we have
$$p(y,z)=\int p(z|x)p(x,y)dx$$
Hence
$$p(y,z)\rvert_{p(z|x)+\epsilon h(z|x)}=p(y,z)\rvert_{p(z|x)}+\epsilon\int h(z|x)p(x,y)dx$$

Then expanding $F_2[p(z|x)+\epsilon h(z|x)]$ to the second order of $\epsilon$, we have

\begin{equation*}
\begin{aligned}
&F_2[p(z|x)+\epsilon h(z|x)]\\
&=\int dx dy dz p(x,y)p(z|x)\left(1+\epsilon\frac{ h(z|x)}{p(z|x)}\right)\log\frac{p(y,z)\left(1+\epsilon\frac{\int h(z|x')p(x',y)dx'}{p(y,z)}\right)}{p(y)p(z)(1+\epsilon\frac{\int h(z|x'')p(x'')dx''}{p(z)})}\\
&=\int dx dy dz p(x,y)p(z|x)\left(1+\epsilon\frac{ h(z|x)}{p(z|x)}\right)\bigg[\log\frac{p(y,z)}{p(y)p(z)}+\epsilon\bigg(\frac{\int h(z|x')p(x',y)dx'}{p(y,z)}-\frac{\int h(z|x')p(x')dx'}{p(z)}\bigg)\\
&+\epsilon^2\bigg[\bigg(\frac{\int h(z|x')p(x')dx'}{p(z)}\bigg)^2-\frac{\int h(z|x')p(x',y)dx'}{p(y,z)}\frac{\int h(z|x'')p(x'')dx''}{p(z)}-\frac{1}{2}\bigg(\frac{\int h(z|x')p(x',y)dx'}{p(y,z)}-\frac{\int h(z|x')p(x')dx'}{p(z)}\bigg)^2\bigg]\\
&+\O(\epsilon^3)\\
\end{aligned}
\end{equation*}

Collecting the first order terms of $\epsilon$, we have
\begin{equation*}
\begin{aligned}
&\delta F_2[p(z|x)]\\
&=\epsilon\int dx dy dz p(x,y)h(z|x)\log\frac{p(y,z)}{p(y)p(z)}+\epsilon\int dx dy dz p(x,y)p(z|x)\frac{\int h(z|x')p(x',y)dx'}{p(y,z)}\\
&-\epsilon\int dx dy dz p(x,y)p(z|x)\frac{\int h(z|x')p(x')dx'}{p(z)}\\
&=\epsilon\int dx dy dz p(x,y)h(z|x)\log\frac{p(y,z)}{p(y)p(z)}+\epsilon\int dx'dy dz h(z|x')p(x',y)-\epsilon\int dz h(z|x')p(x')dx'\\
&=\epsilon\int dx dy dz p(x,y)h(z|x)\log\frac{p(z|y)}{p(z)}
\end{aligned}
\end{equation*}

Collecting the second order terms of $\epsilon$, we have

\begin{equation*}
\begin{aligned}
&\delta^2 F_2[p(z|x)]\\
&=\epsilon^2\int dx dy dz p(x,y)p(z|x)\bigg[\bigg(\frac{\int h(z|x')p(x')dx'}{p(z)}\bigg)^2-\frac{\int h(z|x')p(x',y)dx'}{p(y,z)}\frac{\int h(z|x'')p(x'')dx''}{p(z)}\bigg]\\
&-\frac{\epsilon^2}{2}\int dx dy dz p(x,y)p(z|x)\bigg(\frac{\int h(z|x')p(x',y)dx'}{p(y,z)}-\frac{\int h(z|x')p(x')dx'}{p(z)}\bigg)^2\\
&+\epsilon^2\int dx dy dz p(x,y)p(z|x)\frac{ h(z|x)}{p(z|x)}\bigg(\frac{\int h(z|x')p(x',y)dx'}{p(y,z)}-\frac{\int h(z|x')p(x')dx'}{p(z)}\bigg)\\
&=\frac{\epsilon^2}{2}\int dx dx' dy dz\frac{p(x,y)p(x',y)}{p(y,z)}h(z|x)h(z|x')-\frac{\epsilon^2}{2}\int dx dx' dz \frac{p(x)p(x')}{p(z)}h(z|x)h(z|x')
\end{aligned}
\end{equation*}

Finally, we have
\begin{equation}
\label{eq:delta_IB}
\begin{aligned}
\delta \IB_\beta[p(z|x)]&=\delta F_1[p(z|x)]-\beta \cdot\delta F_2[p(z|x)]\\
&=\epsilon\bigg(\int dx dz p(x) h(z|x)\log\frac{p(z|x)}{p(z)}-\beta \int dx dy dz p(x,y)h(z|x)\log\frac{p(z|y)}{p(z)}\bigg)
\end{aligned}
\end{equation}

\begin{equation*}
\begin{aligned}
\delta^2 \IB_\beta[p(z|x)]=&\delta^2 F_1[p(z|x)]-\beta \cdot\delta^2 F_2[p(z|x)]\\
=&\frac{\epsilon^2}{2}\int dxdz\frac{p(x)^2}{p(x,z)}h(z|x)^2-\frac{\epsilon^2}{2}\int dx dx' dz \frac{p(x)p(x')}{p(z)}h(z|x)h(z|x')\\
&-\beta\epsilon^2\bigg[\frac{1}{2}\int dx dx' dy dz\frac{p(x,y)p(x',y)}{p(y,z)}h(z|x)h(z|x')-\frac{1}{2}\int dx dx' dz \frac{p(x)p(x')}{p(z)}h(z|x)h(z|x')\bigg]\\
=&\frac{\epsilon^2}{2}\bigg[\int dxdz\frac{p(x)^2}{p(x,z)}h(z|x)^2\\
&-\beta\int dx dx' dy dz\frac{p(x,y)p(x',y)}{p(y,z)}h(z|x)h(z|x')+(\beta-1)\int dx dx' dz \frac{p(x)p(x')}{p(z)}h(z|x)h(z|x')\bigg]
\end{aligned}
\end{equation*}

Absorb $\epsilon$ into $h(z|x)$, we get rid of the $\epsilon$ factor and obtain the final expression in Lemma \ref{lemma:first_second_variation_IB}.

\end{proof}

\section{Proof of Lemma \ref{lemma:stationary}}
\label{app:stationary}

\begin{proof}
Using Lemma \ref{lemma:first_second_variation_IB}, we have
\begin{equation*}
\begin{aligned}
\delta \IB_\beta[p(z|x)]=\int dx dz p(x) h(z|x)\log\frac{p(z|x)}{p(z)}-\beta \int dx dy dz p(x,y)h(z|x)\log\frac{p(z|y)}{p(z)}
\end{aligned}
\end{equation*}
Let $p(z|x)=p(z)$ (the trivial representation), we have that $\log\frac{p(z|x)}{p(z)}\equiv0$. Therefore, the two integrals are both 0. Hence,
\begin{equation*}
\begin{aligned}
\delta \IB_\beta[p(z|x)]\big\rvert_{p(z|x)=p(z)}\equiv0
\end{aligned}
\end{equation*}
Therefore, the $p(z|x)=p(z)$ is a stationary solution for $\IB_\beta[p(z|x)]$.

\end{proof}

\section{Proof of Theorem \ref{thm:suff_2}}
\label{app:suff_2}

\begin{proof}

Firstly, from the necessary condition of $\beta>1$ in Section \ref{sec:learnability}, we have that any sufficient condition for $\IB_\beta$-learnability should be able to deduce $\beta>1$.

Now using Theorem \ref{thm:suff_1}, a sufficient condition for $(X,Y)$ to be $\IB_\beta$-learnable is that there exists $h(z|x)$ with $\int h(z|x)dx=0$ such that $\delta^2\IB_\beta[p(z|x)]<0$ at $p(z|x)=p(x)$.

At the trivial representation, $p(z|x)=p(z)$ and hence $p(x,z)=p(x)p(z)$. Due to the Markov chain $Z\gets X\leftrightarrow Y$, we have $p(y,z)=p(y)p(z)$. Substituting them into the $\delta^2\IB_\beta[p(z|x)]$ in Lemma \ref{lemma:first_second_variation_IB}, the condition becomes: there exists $h(z|x)$ with $\int h(z|x)dz=0$, such that
\begin{equation}
\label{eq:suff_2_app1}
\begin{aligned}
&0>\delta^2 \IB_\beta[p(z|x)]=\\
&\frac{1}{2}\bigg[\int dxdz\frac{p(x)^2}{p(x)p(z)}h(z|x)^2-\beta\int dx dx' dy dz\frac{p(x,y)p(x',y)}{p(y)p(z)}h(z|x)h(z|x')+(\beta-1)\int dx dx' dz \frac{p(x)p(x')}{p(z)}h(z|x)h(z|x')\bigg]
\end{aligned}
\end{equation}

Rearranging terms and simplifying, we have
\begin{equation*}
\begin{aligned}
\int \frac{dz}{p(z)}G[h(z|x)]=\int \frac{dz}{p(z)}\bigg[\int dx h(z|x)^2p(x)-\beta\int\frac{dy}{p(y)}\bigg(\int dx h(z|x)p(x)p(y|x)\bigg)^2+(\beta-1)\bigg(\int dx h(z|x)p(x)\bigg)^2\bigg]<0
\end{aligned}
\end{equation*}

where 
$$G[h(x)]=\int dx h(x)^2p(x)-\beta\int\frac{dy}{p(y)}\bigg(\int dx h(x)p(x)p(y|x)\bigg)^2+(\beta-1)\bigg(\int dx h(x)p(x)\bigg)^2$$

Now we prove that the condition that $\exists h(z|x)$ s.t. $\int\frac{dz}{p(z)}G[h(z|x)]<0$ is equivalent to the condition that $\exists h(x)$ s.t. $G[h(x)]<0$.

If $\forall h(z|x)$, $G[h(z|x)]\ge0$, then we have $\forall h(z|x)$, $\int \frac{dz}{p(z)}G[h(z|x)]\ge0$. Therefore, if $\exists h(z|x)$ s.t. $\int\frac{dz}{p(z)}G[h(z|x)]<0$, we have that $\exists h(z|x)$ s.t. $G[h(z|x)]<0$. 
Since the functional $G[h(z|x)]$ does not contain integration over $z$, we can treat the $z$ in $G[h(z|x)]$ as a parameter and we have that $\exists h(x)$ s.t. $G[h(x)]<0$.

Conversely, if there exists an certain function $h(x)$ such that $G[h(x)]<0$, we can find some $h_2(z)$ such that $\int h_2(z)dz=0$ and $\int\frac{h^2_2(z)}{p(z)}dz>0$, and let $h_1(z|x)=h(x)h_2(z)$. Now we have

$$\int\frac{dz}{p(z)}G[h(z|x)]=\int\frac{h_2^2(z)dz}{p(z)}G[h(x)]=G[h(x)]\int\frac{h_2^2(z)dz}{p(z)}<0$$

In other words, the condition Eq. (\ref{eq:suff_2_app1}) is equivalent to requiring that there exists an $h(x)$ such that $G[h(x)]<0$
. Hence, a sufficient condition for $\IB_\beta$-learnability is that there exists an $h(x)$ such that

\begin{equation}
\label{eq:suff_2_app_0}
\begin{aligned}
G[h(x)]=\int dx h(x)^2p(x)-\beta\int\frac{dy}{p(y)}\bigg(\int dx h(x)p(x)p(y|x)\bigg)^2+(\beta-1)\bigg(\int dx h(x)p(x)\bigg)^2<0
\end{aligned}
\end{equation}

When $h(x)=C=\text{constant}$ in the entire input space $\X$, Eq. (\ref{eq:suff_2_app_0}) becomes:

\begin{equation*}
\begin{aligned}
C^2-\beta C^2+(\beta-1)C^2<0
\end{aligned}
\end{equation*}

which cannot be true. Therefore, $h(x)=\text{constant}$ cannot satisfy Eq. (\ref{eq:suff_2_app_0}).

Rearranging terms and simplifying, and note that $\big[\int dx h(x)p(x)\big]^2>0$ due to $h(x)\not\equiv 0=\text{constant}$, we have

\begin{equation}
\label{eq:suff_2_app_1}
\begin{aligned}
\beta\bigg[\frac{\int\frac{dy}{p(y)}\big(\int dx h(x)p(x)p(y|x)\big)^2}{\big(\int dx h(x)p(x)\big)^2}-1\bigg]>\frac{\int dx h(x)^2 p(x)}{\big(\int dx h(x)p(x)\big)^2}-1
\end{aligned}
\end{equation}

For the R.H.S. of Eq. (\ref{eq:suff_2_app_1}), let us show that it is greater than 0. Using Cauchy-Schwarz inequality: $\langle u,u\rangle \langle v,v\rangle\geq\langle u,v\rangle^2$, and setting $u(x)=h(x)\sqrt{p(x)}$, $v(x)=\sqrt{p(x)}$, and defining the inner product as $\langle u,v\rangle=\int u(x)v(x)dx$. We have
$$\frac{\int dx h(x)^2 p(x)}{\big(\int dx h(x)p(x)\big)^2}\geq\frac{1}{\int p(x)dx}=1$$

It attains equality when $\frac{u(x)}{v(x)}=h(x)$ is constant. Since $h(x)$ cannot be constant, we have that the R.H.S. of Eq. (\ref{eq:suff_2_app_1}) is greater than 0. 

For the L.H.S. of Eq. (\ref{eq:suff_2_app_1}), due to the necessary condition that $\beta>0$, if $\bigg[\frac{\int\frac{dy}{p(y)}\big(\int dx h(x)p(x)p(y|x)\big)^2}{\big(\int dx h(x)p(x)\big)^2}-1\bigg]\leq0$, Eq. (\ref{eq:suff_2_app_1}) cannot hold. Then the $h(x)$ such that Eq. (\ref{eq:suff_2_app_1}) holds is for those that satisfies

$$\frac{\int\frac{dy}{p(y)}\big(\int dx h(x)p(x)p(y|x)\big)^2}{\big(\int dx h(x)p(x)\big)^2}-1>0$$

i.e.
\begin{equation*}
\begin{aligned}
\int dy p(y)\bigg(\int dx p(x|y)h(x)\bigg)^2>\bigg(\int dx p(x)h(x)\bigg)^2
\end{aligned}
\end{equation*}

We see this constraint contains the requirement that $h(x)\not\equiv \text{constant}$.

Written in the form of expectations, we have

\begin{equation}
\label{eq:suff_2_app_2}
\begin{aligned}
\E_{y\sim p(y)}\bigg[\bigg(\E_{x\sim p(x|y)} [h(x)]\bigg)^2\bigg]>\left(\E_{x\sim p(x)}[h(x)]\right)^2
\end{aligned}
\end{equation}

Since the square function is convex, using Jensen's inequality on the outer expectation on the L.H.S. of Eq. (\ref{eq:suff_2_app_2}), we have

\begin{equation*}
\begin{aligned}
\E_{y\sim p(y)}\bigg[\bigg(\E_{x\sim p(x|y)} [h(x)]\bigg)^2\bigg]\ge\bigg(\E_{y\sim p(y)}\bigg[\E_{x\sim p(x|y)} [h(x)]\bigg]\bigg)^2 = \left(\E_{x\sim p(x)}[h(x)]\right)^2
\end{aligned}
\end{equation*}

The equality holds iff $\E_{x\sim p(x|y)} [h(x)]$ is constant w.r.t. $y$, i.e. $Y$ is independent of $X$. Therefore, in order for Eq. (\ref{eq:suff_2_app_2}) to hold, we require that $Y$ is not independent of $X$.

Using Jensen's inequality on the innter expectation on the L.H.S. of Eq. (\ref{eq:suff_2_app_2}), we have

\begin{equation}
\label{eq:suff_2_app_22}
\begin{aligned}
\E_{y\sim p(y)}\bigg[\bigg(\E_{x\sim p(x|y)} [h(x)]\bigg)^2\bigg]\le\E_{y\sim p(y)}\big[\E_{x\sim p(x|y)} [h(x)^2]\big] = \E_{x\sim p(x)}[h(x)^2]
\end{aligned}
\end{equation}

The equality holds when $h(x)$ is a constant. Since we require that $h(x)$ is not a constant, we have that the equality cannot be reached.

Under the constraint that $Y$ is not independent of $X$, we can divide both sides of Eq. \ref{eq:suff_2_app_0}, and obtain the condition: there exists an $h(x)$ such that

\begin{equation*}
\begin{aligned}
\beta>\frac{\frac{\int dx h(x)^2 p(x)}{\big(\int dx h(x)p(x)\big)^2}-1}{\frac{\int\frac{dy}{p(y)}\big(\int dx h(x)p(x)p(y|x)\big)^2}{\big(\int dx h(x)p(x)\big)^2}-1}
\end{aligned}
\end{equation*}

i.e.

\begin{equation*}
\begin{aligned}
\beta>\inf_{h(x)}\frac{\frac{\int dx h(x)^2 p(x)}{\big(\int dx h(x)p(x)\big)^2}-1}{\frac{\int\frac{dy}{p(y)}\big(\int dx h(x)p(x)p(y|x)\big)^2}{\big(\int dx h(x)p(x)\big)^2}-1}
\end{aligned}
\end{equation*}

Written in the form of expectations, we have

\begin{equation*}
\begin{aligned}
\beta>&\inf_{h(x)}\frac{\frac{\E_{x\sim p(x)}[ h(x)^2]}{\left(\E_{x\sim p(x)} [h(x)]\right)^2}-1}{\int\frac{dy}{p(y)}\left(\frac{\E_{x\sim p(x)} [p(y|x)h(x)]}{\E_{x\sim p(x)}[h(x)]}\right)^2-1}\\
=&\inf_{h(x)}\frac{\frac{\E_{x\sim p(x)}[ h(x)^2]}{\left(\E_{x\sim p(x)} [h(x)]\right)^2}-1}{\E_{y\sim p(y)}\bigg[\left(\frac{\E_{x\sim p(x|y)} [h(x)]}{\E_{x\sim p(x)}[h(x)]}\right)^2\bigg]-1}
\end{aligned}
\end{equation*}

We can absorb the constraint Eq. (\ref{eq:suff_2_app_2}) into the above formula, and get

\begin{equation*}
\begin{aligned}
\beta>\inf_{h(x)}\beta_0[h(x)]
\end{aligned}
\end{equation*}

where 
$$\beta_0[h(x)]=\frac{\frac{\E_{x\sim p(x)}[ h(x)^2]}{\left(\E_{x\sim p(x)} [h(x)]\right)^2}-1}{\E_{y\sim p(y)}\bigg[\left(\frac{\E_{x\sim p(x|y)} [h(x)]}{\E_{x\sim p(x)}[h(x)]}\right)^2\bigg]-1}$$

which proves the condition of Theorem \ref{thm:suff_2}. 

Furthermore, from Eq. (\ref{eq:suff_2_app_22}) we have

\begin{equation*}
\begin{aligned}
\beta_0[h(x)]>1
\end{aligned}
\end{equation*}

for $h(x)\not\equiv$ const, which satisfies the necessary condition of $\beta>1$ in Section \ref{sec:learnability}.

\textbf{Proof of lower bound of slope of the Pareto frontier at the origin:} 

Now we prove the second statement of Theorem \ref{thm:suff_2}. Since $\delta I(X;Z)=0$ and $\delta I(Y;Z)=0$ according to Lemma \ref{lemma:stationary}, we have $\left(\frac{\Delta I(Y;Z)}{\Delta I(X;Z)}\right)^{-1}=\left(\frac{\delta^2 I(Y;Z)}{\delta^2 I(X;Z)}\right)^{-1}$. Substituting into the expression of $\delta^2 I(Y;Z)$ and $\delta^2 I(X;Z)$ from Lemma \ref{lemma:first_second_variation_IB}, we have 

\begin{equation*}
\begin{aligned}
&\left(\frac{\Delta I(Y;Z)}{\Delta I(X;Z)}\right)^{-1}\\
&=\left(\frac{\delta^2 I(Y;Z)}{\delta^2 I(X;Z)}\right)^{-1}\\
&=\frac{\frac{\epsilon^2}{2}\int dxdz\frac{p(x)^2}{p(x)p(z)}h(z|x)^2-\frac{\epsilon^2}{2}\int dx dx' dz \frac{p(x)p(x')}{p(z)}h(z|x)h(z|x')}{\frac{\epsilon^2}{2}\int dx dx' dy dz\frac{p(x,y)p(x',y)}{p(y)p(z)}h(z|x)h(z|x')-\frac{\epsilon^2}{2}\int dx dx' dz \frac{p(x)p(x')}{p(z)}h(z|x)h(z|x')}\\
&=\frac{\left(\int dx p(x)h(x)^2-\int dx dx'  p(x)p(x')h(x)h(z|x')\right)\int \frac{h_2(z)^2}{p(z)}dz}{\left(\int dx dx' dy \frac{p(x,y)p(x',y)}{p(y)}h(x)h(z|x')-\int dx dx' p(x)p(x')h(x)h(z|x')\right)\int \frac{h_2(z)^2}{p(z)}dz}\\
&=\frac{\int dx p(x)h(x)^2-\int dx dx'  p(x)p(x')h(x)h(z|x')}{\int dx dx' dy \frac{p(x,y)p(x',y)}{p(y)}h(x)h(z|x')-\int dx dx' p(x)p(x')h(x)h(z|x')}\\
&=\frac{\E_{x\sim p(x)}[ h(x)^2]-\left(\E_{x\sim p(x)} [h(x)]\right)^2}{\E_{y\sim p(y)}\big[\left(\E_{x\sim p(x|y)} [h(x)]\right)^2\big]-\left(\E_{x\sim p(x)} [h(x)]\right)^2}\\
&=\frac{\frac{\E_{x\sim p(x)}[ h(x)^2]}{\left(\E_{x\sim p(x)} [h(x)]\right)^2}-1}{\E_{y\sim p(y)}\bigg[\left(\frac{\E_{x\sim p(x|y)} [h(x)]}{\E_{x\sim p(x)}[h(x)]}\right)^2\bigg]-1}\\
&=\beta_0[h(x)]
\end{aligned}
\end{equation*}

Therefore, $\left(\inf_{h(x)}\beta_0[h(x)]\right)^{-1}$ gives the largest slope of $\Delta I(Y;Z)$ vs. $\Delta I(X;Z)$ for perturbation function of the form $h_1(z|x)=h(x)h_2(z)$ satisfying $\int h_2(z)dz=0$ and $\int\frac{h_2^2(z)}{p(z)}dz>0$, which is a lower bound of slope of $\Delta I(Y;Z)$ vs. $\Delta I(X;Z)$ for all possible perturbation function $h_1(z|x)$. The latter is the slope of the Pareto frontier of the $I(Y;Z)$ vs. $I(X;Z)$ curve at the origin.

\textbf{Inflection point for general $Z$:} If we \emph{do not} assume that $Z$ is at the origin of the information plane, but at some general stationary solution $Z^*$ with $p(z|x)$, we define

\begin{equation*}
\begin{aligned}
\beta^{(2)}[h(x)]&=\left(\frac{\delta^2 I(Y;Z)}{\delta^2 I(X;Z)}\right)^{-1}\\
&=\frac{\frac{\epsilon^2}{2}\int dxdz\frac{p(x)^2}{p(x,z)}h(z|x)^2-\frac{\epsilon^2}{2}\int dx dx' dz \frac{p(x)p(x')}{p(z)}h(z|x)h(z|x')}{\frac{\epsilon^2}{2}\int dx dx' dy dz\frac{p(x,y)p(x',y)}{p(y,z)}h(z|x)h(z|x')-\frac{\epsilon^2}{2}\int dx dx' dz \frac{p(x)p(x')}{p(z)}h(z|x)h(z|x')}\\
&=\frac{\int dxdz\frac{p(x)^2}{p(x,z)}h(z|x)^2-\int dx dx' dz \frac{p(x)p(x')}{p(z)}h(z|x)h(z|x')}{\int dx dx' dy dz\frac{p(x,y)p(x',y)}{p(y,z)}h(z|x)h(z|x')-\int dx dx' dz \frac{p(x)p(x')}{p(z)}h(z|x)h(z|x')}\\
&=\frac{\int \frac{dz}{p(z)}\left[\int dx\frac{p(x)^2}{p(x|z)}h(z|x)^2-\left(\int dx p(x)h(z|x)\right)^2\right]}{\int\frac{dz}{p(z)}\left[\int \frac{dy}{p(y|z)} \left(\int dx p(x,y)h(z|x)\right)^2-\left(\int dx p(x)h(z|x)\right)^2\right]}\\
&=\frac{\int \frac{dz}{p(z)}\left[\frac{\int dx\frac{p(x)^2}{p(x|z)}h(z|x)^2}{\left(\int dx p(x)h(z|x)\right)^2}-1\right]}{\int\frac{dz}{p(z)}\left[\frac{\int \frac{dy}{p(y|z)} \left(\int dx p(x,y)h(z|x)\right)^2}{\left(\int dx p(x)h(z|x)\right)^2}-1\right]}\\
&=\frac{\int dz\left[\frac{\int dx\frac{p(x)}{p(z|x)}h(z|x)^2}{\left(\int dx p(x)h(z|x)\right)^2}-\frac{1}{p(z)}\right]}{\int dz\left[\frac{\int \frac{dy}{p(z|y)p(y)} \left(\int dx p(x,y)h(z|x)\right)^2}{\left(\int dx p(x)h(z|x)\right)^2}-\frac{1}{p(z)}\right]}\\
&=\frac{\int dz\left[\int dx\frac{p(x)}{p(z|x)}h(z|x)^2-\frac{1}{p(z)}(\int dx p(x)h(z|x))^2\right]}{\int dz\left[\int \frac{dy}{p(z|y)p(y)} \left(\int dx p(x,y)h(z|x)\right)^2-\frac{1}{p(z)}\left(\int dx p(x)h(z|x)\right)^2\right]}\\
\end{aligned}
\end{equation*}

which reduces to $\beta_0[h(x)]$ when $p(z|x)=p(z)$.
When
\begin{equation}
\label{eq:general_beta}
\beta>\inf_{h(z|x)}\beta^{(2)}[h(z|x)]
\end{equation}
it becomes a non-stable solution (non-minimum), and we will have other $Z$ that achieves a better $\IB_\beta(X,Y;Z)$ than the current $Z^*$. 

\end{proof}

\section{What IB first learns at its onset of learning}
\label{app:what_first_learns}

In this section, we prove that at the onset of learning, if letting $h(z|x)=h^*(x)h_2(z)$, we have

\begin{equation}
p_\beta(y|x)=p(y)+\epsilon^2 C_z (h^*(x)-\overline{h}^*_x)\int p(x,y)(h^*(x)-\overline{h}^*_x)dx
\end{equation}

where $p_\beta(y|x)$ is the estimated $p(y|x)$ by IB for a certain $\beta$, $h^*(x)=\inf_{h(x)}\beta_0[h(x)]$, $\overline{h}^*_x=\int h^*(x)p(x)dx$, $C_z=\int\frac{h_2^2(z)}{p(z)}dz$ is a constant.

\begin{proof}
In IB, we use $p_\beta(z|x)$ to obtain $Z$ from $X$, then obtain the prediction of $Y$ from $Z$ using $p_\beta(y|z)$. Here we use subscript $\beta$ to denote the probability (density) at the optimum of $\IB_\beta[p(z|x)]$ at a specific $\beta$. We have
\begin{equation*}
\begin{aligned}
p_\beta(y|x)&=\int p_\beta(y|z) p_\beta(z|x)dz \\
&=\int dz \frac{p_\beta(y,z) p_\beta(z|x)}{p_\beta(z)}\\
&=\int dz \frac{p_\beta(z|x)}{p_\beta(z)}\int p(x',y)p_\beta(z|x')dx'
\end{aligned}
\end{equation*}

When we have a small perturbation $\epsilon\cdot h(z|x)$ at the trivial representation, $p_\beta(z|x)=p_{\beta_0}(z)+\epsilon\cdot h(z|x)$, we have $p_\beta(z)=p_{\beta_0}(z)+\epsilon\cdot\int h(z|x'')p(x'')dx''$. Substituting, we have

\begin{equation*}
\begin{aligned}
p_\beta(y|x)&=\int dz \frac{p_{\beta_0}(z)\left(1+\epsilon\cdot\frac{h(z|x)}{p_{\beta_0}(z)}\right)}{p_{\beta_0}(z)\left(1+\epsilon\cdot\frac{\int h(z|x'')p(x'')dx''}{p_{\beta_0}(z)}\right)}\int p(x',y)p_{\beta_0}(z)\left(1+\epsilon\cdot\frac{h(z|x')}{p_{\beta_0}(z)}\right)dx'\\
&=\int dz \frac{1+\epsilon\cdot\frac{h(z|x)}{p_{\beta_0}(z)}}{1+\epsilon\cdot\frac{\int h(z|x'')p(x'')dx''}{p_{\beta_0}(z)}}\int p(x',y)p_{\beta_0}(z)\left(1+\epsilon\cdot\frac{h(z|x')}{p_{\beta_0}(z)}\right)dx'\\
\end{aligned}
\end{equation*}

The $0^{\text{th}}$-order term is $\int dz dx' p(x',y)p_{\beta_0}(z)=p(y)$. The first-order term is

\begin{equation*}
\begin{aligned}
\delta p_\beta(z|x)=&\epsilon\cdot\int dzdx'\left(h(z|x) + h(z|x')-\int h(z|x'')p(x'')dx''\right)p(x',y)\\
=&\epsilon\cdot\int dx' \left(\int dz h(z|x)+\int dz h(z|x')\right)-\epsilon\cdot\int dx'dx''p(x',y)p(x'')\int dz h(z|x'')\\
=&0-0\\
=&0
\end{aligned}
\end{equation*}

since we have $\int h(z|x)dz=0$ for any $x$.

For the second-order term, using $h(z|x)=h^*(x)h_2(z)$ and $C_z=\int\frac{dz}{p_{\beta_0}(z)}h_2^2(z)$, it is
\begin{equation*}
\begin{aligned}
\delta^2 p_\beta(y|x)=&\epsilon^2\cdot\int dz\left(\frac{\int h(z|x'')p(x'')dx''}{p_{\beta_0}(z)}\right)^2 \int p(x',y)p_{\beta_0}(z)dx'\\
&-\epsilon^2\cdot \int dz\frac{h(z|x)\int h(z|x'')p(x'')dx''}{(p_{\beta_0}(z))^2} \int p(x',y)p_{\beta_0}(z)dx'\\
&+\epsilon^2 \int dz\left(h(z|x)-\int h(z|x'')p(x'')dx\right)\int p(x',y)\frac{h(z|x')}{p_{\beta_0}(z)}dx'\\
=&\epsilon^2 C_z\cdot\left(\int h^*(x'')p(x'')dx''\right)^2 p(y)\\
&-\epsilon^2 C_z\cdot h^*(x)\int h^*(x'')p(x'')dx'' p(y)\\
&+\epsilon^2 C_z\cdot h^*(x)\int p(x',y)h^*(x')dx'\\
&-\epsilon^2 C_z\cdot\int h^*(x'')p(x'')dx\int p(x',y)h^*(x')dx'\\
=&\epsilon^2 C_z(h^*(x)-\overline{h}^*_x)\left[\left(\int p(x',y)h^*(x')dx'\right)-\overline{h}^*_x p(y)\right]\\
=&\epsilon^2 C_z(h^*(x)-\overline{h}^*_x)\int p(x',y)\left(h^*(x')-\overline{h}^*_x\right)dx'
\end{aligned}
\end{equation*}

where $\overline{h}^*_x=\int h^*(x)p(x)dx$. Combining everything, we have up to the second order,

\begin{equation*}
p_\beta(y|x)=p(y)+\epsilon^2 C_z (h^*(x)-\overline{h}^*_x)\int p(x,y)(h^*(x)-\overline{h}^*_x)dx
\end{equation*}
\end{proof}

\section{Proof of Theorem \ref{thm:suff_3}}
\label{app:suff_3}

\begin{proof}
According to Theorem \ref{thm:suff_2}, a sufficient condition for $(X,Y)$ to be $\IB_\beta$-learnable is that $X$ and $Y$ are not independent, and
\begin{equation}
\label{eq:suff_3_app_1}
\begin{aligned}
\beta>\inf_{h(x)}\frac{\frac{\E_{x\sim p(x)}[ h(x)^2]}{\left(\E_{x\sim p(x)} [h(x)]\right)^2}-1}{\E_{y\sim p(y)}\bigg[\left(\frac{\E_{x\sim p(x|y)} [h(x)]}{\E_{x\sim p(x)}[h(x)]}\right)^2\bigg]-1}
\end{aligned}
\end{equation}

We can assume a specific form of $h(x)$, and obtain a (potentially stronger) sufficient condition. Specifically, we let

\begin{equation}
\label{eq:suff_3_app_2}
\begin{aligned}
h(x)=\begin{cases}
1, x\in\Omega_x\\
0, \text{otherwise}
\end{cases}
\end{aligned}
\end{equation}

for certain $\Omega_x\subset \X$. Substituting into Eq. (\ref{eq:suff_3_app_2}), we have that a sufficient condition for $(X,Y)$ to be $\IB_\beta$-learnable is

\begin{equation}
\label{eq:suff_3_app_3}
\begin{aligned}
\beta>\inf_{\Omega_x\subset\X}\frac{\frac{p(\Omega_x)}{p(\Omega_x)^2}-1}{\int dy p(y)\left(\frac{\int_{x\in\Omega_x} dx p(x|y) dx}{p(\Omega_x)}\right)^2-1}>0
\end{aligned}
\end{equation}

where $p(\Omega_x)=\int_{x\in\Omega_x}p(x)dx$.

The denominator of Eq. (\ref{eq:suff_3_app_3}) is

\begin{equation*}
\begin{aligned}
&\int dy p(y)\left(\frac{\int_{x\in\Omega_x} dx p(x|y)dx}{p(\Omega_x)}\right)^2-1\\
&=\int dy p(y)\bigg( \frac{p(\Omega_x|y)}{p(\Omega_x)}\bigg)^2-1\\
&=\int dy  \frac{p(y|\Omega_x)^2}{p(y)}-1\\
&=\E_{y\sim p(y|\Omega_x)}\bigg[\frac{p(y|\Omega_x)}{p(y)}-1\bigg]
\end{aligned}
\end{equation*}

Using the inequality $x-1\ge \log\ x$, we have
\begin{equation*}
\begin{aligned}
&\E_{y\sim p(y|\Omega_x)}\bigg[\frac{p(y|\Omega_x)}{p(y)}-1\bigg]\ge \E_{y\sim p(y|\Omega_x)}\bigg[\log\frac{p(y|\Omega_x)}{p(y)}\bigg]\ge 0
\end{aligned}
\end{equation*}

Both equalities hold iff $p(y|\Omega_x)\equiv p(y)$, at which the denominator of Eq. (\ref{eq:suff_3_app_3}) is equal to 0 and the expression inside the infimum diverge, which will not contribute to the infimum. Except this scenario, the denominator is greater than 0. Substituting into Eq. (\ref{eq:suff_3_app_3}), we have that a sufficient condition for $(X,Y)$ to be $\IB_\beta$-learnable is

\begin{equation}
\label{eq:suff_3_app_4}
\begin{aligned}
\beta>\inf_{\Omega_x\subset\X}\frac{\frac{p(\Omega_x)}{p(\Omega_x)^2}-1}{\E_{y\sim p(y|\Omega_x)}\left[\frac{p(y|\Omega_x)}{p(y)}-1\right]}
\end{aligned}
\end{equation}

Since $\Omega_x$ is a subset of $\X$, by the definition of $h(x)$ in Eq. (\ref{eq:suff_3_app_2}), $h(x)$ is not a constant in the entire $\X$. Hence the numerator of Eq. (\ref{eq:suff_3_app_4}) is positive. Since its denominator is also positive, we can then neglect the ``$>0$", and obtain the condition in Theorem \ref{thm:suff_3}.

Since the $h(x)$ used in this theorem is a subset of the $h(x)$ used in Theorem \ref{thm:suff_2}, the infimum for Eq. (\ref{eq:suff_3}) is greater than or equal to the infimum in Eq. (\ref{eq:suff_2}). Therefore, according to the second statement of Theorem \ref{thm:suff_2}, we have that the $\left(\inf_{\Omega_x\subset\X}\beta_0(\Omega_x)\right)^{-1}$ is also a lower bound of the slope for the Pareto frontier of $I(Y;Z)$ vs. $I(X;Z)$ curve.

Now we prove that the condition Eq. (\ref{eq:suff_3}) is invariant to invertible mappings of $X$. In fact, if $X'=g(X)$ is a uniquely invertible map (if $X$ is continuous, $g$ is additionally required to be continuous), let $\X'=\{g(x)|x\in\Omega_x\}$, and denote $g(\Omega_x)\equiv\{g(x)|x\in\Omega_x\}$ for any $\Omega_x\subset \X$, we have $p(g(\Omega_x))=p(\Omega_x)$, and $p(y|g(\Omega_x))=p(y|\Omega_x)$.  Then for dataset $(X,Y)$, let $\Omega_x'=g(\Omega_x)$, we have
\begin{equation}
\begin{aligned}
&\frac{\frac{1}{p(\Omega_x')} - 1}{\E_{y \sim p(y|\Omega_x')} \bigg[ \frac{p(y|\Omega_x')}{p(y)} - 1 \bigg]}=\frac{\frac{1}{p(\Omega_x)} - 1}{\E_{y \sim p(y|\Omega_x)} \bigg[ \frac{p(y|\Omega_x)}{p(y)} - 1 \bigg]}
\end{aligned}
\end{equation}

Additionally we have $\X'=g(\X)$. Then

\begin{equation}
\begin{aligned}
\label{eq:suf_2_app_3}
\inf_{\Omega_x'\subset \X'}\frac{\frac{1}{p(\Omega_x')} - 1}{\E_{y \sim p(y|\Omega_x')} \bigg[ \frac{p(y|\Omega_x')}{p(y)} - 1 \bigg]}=\inf_{\Omega_x\subset \X}\frac{\frac{1}{p(\Omega_x)} - 1}{\E_{y \sim p(y|\Omega_x)} \bigg[ \frac{p(y|\Omega_x)}{p(y)} - 1 \bigg]}
\end{aligned}
\end{equation}

For dataset $(X',Y)=(g(X),Y)$, applying Theorem \ref{thm:suff_3} we have that a sufficient condition for it to be $\IB_\beta$-learnable is 

\begin{equation}
\begin{aligned}
\beta>\inf_{\Omega_x'\subset \X'}\frac{\frac{1}{p(\Omega_x')} - 1}{\E_{y \sim p(y|\Omega_x')} \bigg[ \frac{p(y|\Omega_x')}{p(y)} - 1 \bigg]}=\inf_{\Omega_x\subset \X}\frac{\frac{1}{p(\Omega_x)} - 1}{\E_{y \sim p(y|\Omega_x)} \bigg[ \frac{p(y|\Omega_x)}{p(y)} - 1 \bigg]}
\end{aligned}
\end{equation}

where the equality is due to Eq. (\ref{eq:suf_2_app_3}). Comparing with the condition for $\IB_\beta$-learnability for $(X,Y)$ (Eq. (\ref{eq:suff_3})), we see that they are the same. Therefore, the condition given by Theorem \ref{thm:suff_3} is invariant to invertible mapping of $X$.

\end{proof}

\section{Proof of Corollary \ref{corollary:suff_3_class_conditional} and Corollary \ref{corollary:suff_3_2}}
\label{app:corollaries}

\subsection{Proof of Corollary \ref{corollary:suff_3_class_conditional}}

\begin{proof}
We use Theorem \ref{thm:suff_3}. Let $\Omega_x$ contain all elements $x$ whose true class is $y^*$ for some certain $y^*$, and 0 otherwise. Then we obtain a (potentially stronger) sufficient condition. Since the probability $p(y|y^*,x)=p(y|y^*)$ is class-conditional, we have

\begin{equation*}
\begin{aligned}
&\inf_{\Omega_x\subset\X}\frac{\frac{1}{p(\Omega_x)} - 1}{\E_{y \sim p(y|\Omega_x)} \bigg[ \frac{p(y|\Omega_x)}{p(y)} - 1 \bigg]}\\
=&\inf_{y^*}\frac{\frac{1}{p(y^*)} - 1}{\E_{y \sim p(y|y^*)} \bigg[ \frac{p(y|y^*)}{p(y)} - 1 \bigg]}
\end{aligned}
\end{equation*}

By requiring $\beta>\inf_{y^*}\frac{\frac{1}{p(y^*)} - 1}{\E_{y \sim p(y|y^*)} \big[ \frac{p(y|y^*)}{p(y)} - 1 \big]}$, we obtain a sufficient condition for $\IB_\beta$ learnability.
\end{proof}

\subsection{Proof of Corollary \ref{corollary:suff_3_2}}

\begin{proof}
We again use Theorem \ref{thm:suff_3}. Since $Y$ is a deterministic function of $X$, let $Y=f(X)$. By the assumption that $Y$ contains at least one value $y$ such that its probability $p(y)>0$, we let $\Omega_x$ contain only $x$ such that $f(x)=y$. Substituting into Eq. (\ref{eq:suff_3}), we have

\begin{equation*}
\begin{aligned}
&\frac{\frac{1}{p(\Omega_x)} - 1}{\E_{y \sim p(y|\Omega_x)} \bigg[ \frac{p(y|\Omega_x)}{p(y)} - 1 \bigg]}\\
=&\frac{\frac{1}{p(y)} - 1}{\E_{y \sim p(y|\Omega_x)} \bigg[ \frac{1}{p(y)} - 1 \bigg]}\\
=&\frac{\frac{1}{p(y)} - 1}{ \frac{1}{p(y)} - 1 }\\
=&1
\end{aligned}
\end{equation*}
\end{proof}
Therefore, the sufficient condition becomes $\beta>1$.

\section{$\beta_0$, hypercontractivity coefficient, contraction coefficient, $\beta_0[h(x)]$, and maximum correlation}
\label{app:maximum_corr}

In this section, we prove the relations between the IB-Learnability threshold $\beta_0$, the hypercontractivity coefficient $\xi(X;Y)$, the contraction coefficient $\eta_\text{KL}(p(y|x),p(x))$, $\beta_0[h(x)]$ in Eq. (\ref{eq:suff_2}), and maximum correlation $\rho_m(X,Y)$, as follows:

\begin{align}
\frac{1}{\beta_0} = \xi(X;Y)=\eta_\text{KL}(p(y|x),p(x))\ge \sup_{h(x)}\frac{1}{\beta_0[h(x)]} = \rho_m^2(X;Y)
\end{align}

\begin{proof}
The hypercontractivity coefficient $\xi$ is defined as \citep{anantharam2013maximal}:
$$\xi(X;Y)\equiv\sup_{Z-X-Y}\frac{I(Y;Z)}{I(X;Z)}$$

By our definition of IB-learnability, ($X$, $Y$) is IB-Learnable iff there exists $Z$ obeying the Markov chain $Z-X-Y$, such that

$$I(X;Z)-\beta\cdot I(Y;Z)<0=IB_\beta(X,Y;Z)|_{p(z|x)=p(z)}$$

Or equivalently there exists $Z$ obeying the Markov chain $Z-X-Y$ such that 
\begin{equation}
\label{eq:relation_11}
0<\frac{1}{\beta}<\frac{I(Y;Z)}{I(X;Z)}
\end{equation}

By Theorem \ref{thm:beta_monotonic}, the IB-Learnability region for $\beta$ is $(\beta_0, +\infty)$, or equivalently the IB-Learnability region for $1/\beta$ is
\begin{equation}
\label{eq:relation_12}
0<\frac{1}{\beta}<\frac{1}{\beta_0}
\end{equation}

Comparing Eq. (\ref{eq:relation_11}) and Eq. (\ref{eq:relation_12}), we have that

\begin{equation}
\label{eq:relation_13}
\frac{1}{\beta_0} = \sup_{Z-X-Y}\frac{I(Y;Z)}{I(X;Z)}=\xi(X;Y)
\end{equation}

In \citet{anantharam2013maximal}, the authors prove that 

\begin{equation}
\xi(X;Y) =\eta_\text{KL}(p(y|x),p(x))
\end{equation}

where the contraction coefficient $\eta_\text{KL}(p(y|x),p(x))$ is defined as

\begin{equation*}
\eta_\text{KL}(p(y|x),p(x))=\sup_{r(x)\neq p(x)}\frac{\mathbb{D}_\text{KL}(r(y)||p(y))}{\mathbb{D}_\text{KL}(r(x)||p(x))}
\end{equation*}

where $p(y)=\mathbb{E}_{x\sim p(x)}[p(y|x)]$ and $r(y)=\mathbb{E}_{x\sim r(x)}[p(y|x)]$.
Treating $p(y|x)$ as a channel, the contraction coefficient measures how much the two distributions $r(x)$ and $p(x)$ becomes ``nearer" (as measured by the KL-divergence) after passing through the channel.

In \citet{anantharam2013maximal}, the authors also provide a counterexample to an earlier result by \citet{erkip1998efficiency} that incorrectly proved $\xi(X;Y)=\rho_m^2(X;Y)$.
In the specific counterexample \citet{anantharam2013maximal} design, $\xi(X;Y)>\rho_m^2(X;Y)$.

The maximum correlation is defined as
$\rho_m(X;Y)\equiv\max_{f,g} \mathbb{E}[f(X)g(Y)]$ where $f(X)$ and $g(Y)$ are real-valued random variables such that $\mathbb{E}[f(X)]=\mathbb{E}[g(Y)]=0$ and $\mathbb{E}[f^2(X)]=\mathbb{E}[g^2(Y)]=1$ \citep{hirschfeld1935connection,gebelein1941statistische}.

Now we prove $\xi(X;Y)\ge\rho_m^2(X;Y)$, based on Theorem \ref{thm:suff_2}. To see this, we use the alternate characterization of $\rho_m(X;Y)$ by \citet{renyi1959measures}:

\begin{equation}
\label{eq:renyi}
\rho_m^2(X;Y)=\max_{f(X):\mathbb{E}[f(X)]=0,\mathbb{E}[f^2(X)]=1}{\mathbb{E}[\left(\mathbb{E}[f(X)|Y]\right)^2]}
\end{equation}

Denoting $\overline{h}=\mathbb{E}_{p(x)}[h(x)]$, we can transform  $\beta_0[h(x)]$ in Eq. (\ref{eq:suff_2}) as follows:

\begin{equation*}
\begin{aligned}
\beta_0[h(x)]&=\frac{\E_{x \sim p(x)} [h(x)^2] - \left(\E_{x\sim p(x)} [h(x)]\right)^2}{\E_{y \sim p(y)}\left[\left(\E_{x \sim p(x|y)} [h(x)]\right)^2\right] - \left(\E_{x\sim p(x)} [h(x)]\right)^2}\\
&=\frac{\E_{x \sim p(x)} [h(x)^2] - \overline{h}^2}{\E_{y \sim p(y)}\left[\left(\E_{x \sim p(x|y)} [h(x)]\right)^2\right] - \overline{h}^2}\\
&=\frac{\E_{x \sim p(x)} [(h(x)-\overline{h})^2]}{\E_{y \sim p(y)}\left[\left(\E_{x \sim p(x|y)} [h(x)-\overline{h}]\right)^2\right]}\\
&=\frac{1}{\E_{y \sim p(y)}\left[\left(\E_{x \sim p(x|y)} [f(x)]\right)^2\right]}\\
&=\frac{1}{\mathbb{E}[\left(\mathbb{E}[f(X)|Y]\right)^2]}
\end{aligned}
\end{equation*}

where we denote $f(x)=\frac{h(x)-\overline{h}}{\left(\E_{x \sim p(x)} [(h(x)-\overline{h})^2]\right)^{1/2}}$, so that $\mathbb{E}[f(X)]=0$ and $\mathbb{E}[f^2(X)]=1$.

Combined with Eq. (\ref{eq:renyi}), we have

\begin{equation}
\label{eq:relations_14}
\sup_{h(x)}\frac{1}{\beta_0[h(x)]}=\rho_m^2(X;Y) 
\end{equation}

Our Theorem \ref{thm:suff_2} states that

\begin{equation}
\label{eq:relations_15}
\sup_{h(x)}\frac{1}{\beta_0[h(x)]}\leq\frac{1}{\beta_0}
\end{equation}

Combining Eqs. (\ref{eq:relation_12}), (\ref{eq:relations_14}) and Eq. (\ref{eq:relations_15}), we have

\begin{equation}
\label{eq:relations_16}
\rho_m^2(X;Y)\leq\xi(X;Y)
\end{equation}

In summary, the relations among the quantities are:

\begin{equation}
\label{eq:relations_summary}
\frac{1}{\beta_0}=\xi(X;Y)=\eta_\text{KL}(p(y|x),p(x))\ge
\sup_{h(x)}\frac{1}{\beta_0[h(x)]}=\rho_m^2(X;Y)
\end{equation}

\end{proof}

\section{Experiment Details}
\label{app:experiment}

We use the Variational Information Bottleneck (VIB) objective from \cite{alemi2016deep}.
For the synthetic experiment, the latent $Z$ has dimension of 2.
The encoder is a neural net with 2 hidden layers, each of which has 128 neurons with ReLU activation. 
The last layer has linear activation and 4 output neurons; the first two parameterize the mean of a Gaussian and the last two parameterize the log variance.
The decoder is a neural net with 1 hidden layer with 128 neurons and ReLU activation. 
Its last layer has linear activation and outputs the logit for the class labels. 
It uses a mixture of Gaussian prior with 500 components (for the experiment with class overlap, 256 components), each of which is a 2D Gaussian with learnable mean and log variance, and the weights for the components are also learnable. 
For the MNIST experiment, the architecture is mostly the same, except the following: (1) for $Z$, we let it have dimension of 256. 
For the prior, we use standard Gaussian with diagonal covariance matrix. 

For all experiments, we use Adam (\cite{kingma2014adam}) optimizer with default parameters. 
We do not add any explicit regularization. 
We use learning rate of $10^{-4}$ and have a learning rate decay of $\frac{1}{1+0.01 \times \text{epoch}}$. 
We train in total $2000$ epochs with mini-batch size of 500. 

For estimation of the observed $\beta_{0}$ in Fig. \ref{fig:gauss_noise_beta}, in the $I(X;Z)$ vs. $\beta_i$ curve ($\beta_i$ denotes the $i^{\text{th}}$ $\beta$), we take the mean and standard deviation of $I(X;Z)$ for the lowest 5 $\beta_i$ values, denoting as $\mu_\beta$, $\sigma_\beta$ ($I(Y;Z)$ has similar behavior, but since we are minimizing $I(X;Z)-\beta \cdot I(Y;Z)$, the onset of nonzero $I(X;Z)$ is less prone to noise).
When $I(X;Z)$ is greater than $\mu_\beta$ + 3$\sigma_\beta$, we regard it as learning a non-trivial representation, and take the average of $\beta_i$ and $\beta_{i-1}$ as the experimentally estimated onset of learning. 
We also inspect manually and confirm that it is consistent with human intuition.

For estimating $\beta_{0}$ using Alg. \ref{alg:estimating_beta},  at step 6 we use the following discrete search algorithm.
We fix $i_\text{left}=1$ and gradually narrow down the range $[a,b]$ of $i_\text{right}$, starting from $[1,N]$.
At each iteration, we set a tentative new range $[a',b']$, where $a'=0.8a+0.2b$, $b'=0.2a+0.8b$, and calculate $\tilde{\beta}_{0,a'}=\textbf{Get}\boldsymbol{\beta}(P_{y|x},p_y,\Omega_{a'})$, $\tilde{\beta}_{0,b'}=\textbf{Get}\boldsymbol{\beta}(P_{y|x},p_y,\Omega_{b'})$ where $\Omega_{a'} =\{1,2,...a'\}$ and $\Omega_{b'} =\{1,2,...b'\}$.
If $\tilde{\beta}_{0,a'}<\tilde{\beta}_{0,a}$, let $a\gets a'$. If $\tilde{\beta}_{0,b'}<\tilde{\beta}_{0,b}$, let $b\gets b'$.
In other words, we narrow down the range of $i_\text{right}$ if we find that the $\Omega$ given by the left or right boundary gives a lower $\tilde{\beta}_0$ value.
The process stops when both $\tilde{\beta}_{0,a'}$ and $\tilde{\beta}_{0,b'}$ stop improving (which we find always happens when $b'=a'+1$), and we return the smaller of the final $\tilde{\beta}_{0,a'}$ and $\tilde{\beta}_{0,b'}$ as $\tilde{\beta}_0$.

For estimation of $p(y|x)$ for (2$'$) Alg. \ref{alg:estimating_beta} and (3$'$) $\hat{\eta}_{\text{KL}}$ for both synthetic and MNIST experiments, we use a 3-layer neuron net where each hidden layer has 128 neurons and ReLU activation. The last layer has linear activation. The objective is cross-entropy loss. We use Adam \citep{kingma2014adam} optimizer with a learning rate of $10^{-4}$, and train for 100 epochs (after which the validation loss does not go down).

For estimating $\beta_0$ via (3$'$) $\hat{\eta}_\text{KL}$ by the algorithm in \citep{kim2017discovering}, we use the code from the GitHub repository provided by the paper\footnote{%
At \href{https://github.com/wgao9/hypercontractivity}{https://github.com/wgao9/hypercontractivity}.
}, using the same $p(y|x)$ employed for (2$'$) Alg. \ref{alg:estimating_beta}. Since our datasets are classification tasks, we use $A_{ij}=p(y_j|x_i)/p(y_j)$ instead of the kernel density for estimating matrix $A$; we take the maximum of 10 runs as estimation of $\mu$.

\subsection{Detailed tables for classification  with  class-conditional  noise}
\label{app:gauss_noise_beta}

In Table~\ref{tab:class_cond_noise} we give the full set of values used for Fig.~\ref{fig:gauss_noise_beta}.
As the noise rate increases, the true $\beta_0$ increases dramatically.
Note that the Observed values are empirical estimates.
Corollary~\ref{corollary:suff_3_class_conditional} and Alg.~\ref{alg:estimating_beta} agree almost perfectly when using the true $p(y|x)$.
$\hat{\eta}_{\KL}$ is somewhat looser, but generally agrees well with the empirical estimates when using the true $p(y|x)$.
However, its estimates become much less accurate when $p(y|x)$ is given by a learned neural network trained on the noisy dataset.
In contrast, Alg.~\ref{alg:estimating_beta} generally gives much better predictions even when using the estimated $p(y|x)$.
Directly optimizing Eq.~\ref{eq:suff_2} on the observed data is always an upper bound, although the bound becomes somewhat looser as the noise becomes extreme.

\begin{table}[p]
\begin{center}
\label{tab:class_cond_noise}
\caption{
Full table of values used to generate Fig.~\ref{fig:gauss_noise_beta}.
}
\vskip 0.1in
\setlength{\tabcolsep}{4pt}  
\begin{tabular}{r | c c c c c c c }

 & & & (2) Alg. \ref{alg:estimating_beta} & (3) $\hat{\eta}_{\KL}$ & & & \\
Noise rate    & Observed & (1) Corollary~\ref{corollary:suff_3_class_conditional} & true $p(y|x)$ & true $p(y|x)$ & (4) Eq.~\ref{eq:suff_2} & (2$'$) Alg.~\ref{alg:estimating_beta} & (3$'$) $\hat{\eta}_{\KL}$ \\
\hline
\hline\noalign{\smallskip}
0.02 &    1.06 &    1.09 &    1.09 &    1.10 &    1.08 &    1.08 &   1.10 \\
0.04 &    1.20 &    1.18 &    1.18 &    1.21 &    1.18 &    1.19 &   1.20 \\
0.06 &    1.26 &    1.29 &    1.29 &    1.33 &    1.30 &    1.31 &   1.33 \\
0.08 &    1.40 &    1.42 &    1.42 &    1.45 &    1.42 &    1.43 &   1.46 \\
0.10 &    1.52 &    1.56 &    1.56 &    1.60 &    1.55 &    1.58 &   1.60 \\
0.12 &    1.70 &    1.73 &    1.73 &    1.78 &    1.71 &    1.73 &   1.77 \\
0.14 &    1.99 &    1.93 &    1.93 &    1.99 &    1.90 &    1.91 &   1.95 \\
0.16 &    2.04 &    2.16 &    2.16 &    2.24 &    2.15 &    2.15 &   2.16 \\
0.18 &    2.41 &    2.44 &    2.44 &    2.49 &    2.43 &    2.42 &   2.49 \\
0.20 &    2.74 &    2.78 &    2.78 &    2.86 &    2.76 &    2.77 &   2.71 \\
0.22 &    3.15 &    3.19 &    3.19 &    3.29 &    3.19 &    3.21 &   3.29 \\
0.24 &    3.75 &    3.70 &    3.70 &    3.83 &    3.71 &    3.75 &   3.72 \\
0.26 &    4.40 &    4.34 &    4.34 &    4.48 &    4.35 &    4.31 &   4.17 \\
0.28 &    5.16 &    5.17 &    5.17 &    5.37 &    5.12 &    4.98 &   4.55 \\
0.30 &    6.34 &    6.25 &    6.25 &    6.49 &    6.24 &    6.03 &   5.58 \\
0.32 &    8.06 &    7.72 &    7.72 &    8.02 &    7.63 &    7.19 &   7.33 \\
0.34 &    9.77 &    9.77 &    9.77 &   10.13 &    9.74 &    8.95 &   7.37 \\
0.36 &   12.58 &   12.76 &   12.76 &   13.21 &   12.51 &   11.11 &  10.09 \\
0.38 &   16.91 &   17.36 &   17.36 &   17.96 &   16.97 &   14.55 &  10.49 \\
0.40 &   24.66 &   25.00 &   25.00 &   25.99 &   25.01 &   20.36 &  17.27 \\
0.42 &   39.08 &   39.06 &   39.06 &   40.85 &   39.48 &   30.12 &  10.89 \\
0.44 &   64.82 &   69.44 &   69.44 &   71.80 &   76.48 &   51.95 &  21.95 \\
0.46 &  163.07 &  156.25 &  156.26 &  161.88 &  173.15 &  114.57 &  21.47 \\
0.48 &  599.45 &  625.00 &  625.00 &  651.47 &  838.90 &  293.90 &   8.69 \\
\hline
\end{tabular}
\end{center}
\end{table}

\begin{table}[p]
\begin{center}
\label{tab:cifar_confusion}
\caption{
Class confusion matrix used in CIFAR10 experiments.
The value in row $i$, column $j$ means for class $i$, the probability of labeling it as class $j$. The mean confusion across the classes is 20\%.
}
\vskip 0.1in
\setlength{\tabcolsep}{4pt}  
\begin{tabular}{r | c c c c c c c c c c }
\small
& Plane & Auto. & Bird & Cat & Deer & Dog & Frog & Horse & Ship & Truck \\ [0.2ex]
\hline
\hline\noalign{\smallskip}
Plane &
0.82232 & 0.00238 & 0.021   & 0.00069 & 0.00108 & 0       & 0.00017 & 0.00019 & 0.1473  & 0.00489 \\ [0.2ex]
Auto. &
0.00233 & 0.83419 & 0.00009 & 0.00011 & 0       & 0.00001 & 0.00002 & 0       & 0.00946 & 0.15379 \\ [0.2ex]
Bird &
0.03139 & 0.00026 & 0.76082 & 0.0095  & 0.07764 & 0.01389 & 0.1031  & 0.00309 & 0.00031 & 0       \\ [0.2ex]
Cat &
0.00096 & 0.0001  & 0.00273 & 0.69325 & 0.00557 & 0.28067 & 0.01471 & 0.00191 & 0.00002 & 0.0001  \\ [0.2ex]
Deer &
0.00199 & 0       & 0.03866 & 0.00542 & 0.83435 & 0.01273 & 0.02567 & 0.08066 & 0.00052 & 0.00001 \\ [0.2ex]
Dog &
0       & 0.00004 & 0.00391 & 0.2498  & 0.00531 & 0.73191 & 0.00477 & 0.00423 & 0.00001 & 0       \\ [0.2ex]
Frog &
0.00067 & 0.00008 & 0.06303 & 0.05025 & 0.0337  & 0.00842 & 0.8433  & 0       & 0.00054 & 0       \\ [0.2ex]
Horse &
0.00157 & 0.00006 & 0.00649 & 0.00295 & 0.13058 & 0.02287 & 0       & 0.83328 & 0.00023 & 0.00196 \\ [0.2ex]
Ship &
0.1288  & 0.01668 & 0.00029 & 0.00002 & 0.00164 & 0.00006 & 0.00027 & 0.00017 & 0.83385 & 0.01822 \\ [0.2ex]
Truck &
0.01007 & 0.15107 & 0       & 0.00015 & 0.00001 & 0.00001 & 0       & 0.00048 & 0.02549 & 0.81273 \\ [0.2ex]
\hline
\end{tabular}
\end{center}
\end{table}

\subsection{CIFAR10 Details}
\label{app:cifar_details}

We trained a deterministic 28x10 wide resnet~\citep{resnet,wideresnet}, using the open source implementation from~\citet{autoaugment}.
However, we extended the final 10 dimensional logits of that model through another 3 layer MLP classifier, in order to keep the inference network architecture identical between this model and the VIB models we describe below.
During training, we dynamically added label noise according to the class confusion matrix in Tab.~\ref{tab:cifar_confusion}.
The mean label noise averaged across the 10 classes is 20\%.
After that model had converged, we used it to estimate $\beta_0$ with Alg.~\ref{alg:estimating_beta}.
Even with 20\% label noise, $\beta_0$ was estimated to be 1.0483.

We then trained 73 different VIB models using the same 28x10 wide resnet architecture for the encoder, parameterizing the mean of a 10-dimensional unit variance Gaussian.
Samples from the encoder distribution were fed to the same 3 layer MLP classifier architecture used in the deterministic model.
The marginal distributions were mixtures of 500 fully covariate 10-dimensional Gaussians, all parameters of which are trained.
The VIB models had $\beta$ ranging from 1.02 to 2.0 by steps of 0.02, plus an extra set ranging from 1.04 to 1.06 by steps of 0.001 to ensure we captured the empirical $\beta_0$ with high precision.

However, this particular VIB architecture does not start learning until $\beta > 2.5$, so none of these models would train as described.\footnote{%
A given architecture trained using maximum likelihood and with no stochastic layers will tend to have higher effective capacity than the same architecture with a stochastic layer that has a fixed but non-trivial variance, even though those two architectures have exactly the same number of learnable parameters.
}
Instead, we started them all at $\beta = 100$, and annealed $\beta$ down to the corresponding target over 10,000 training gradient steps.
The models continued to train for another 200,000 gradient steps after that.
In all cases, the models converged to essentially their final accuracy within 20,000 additional gradient steps after annealing was completed.
They were stable over the remaining $\sim 180,000$ gradient steps.